\documentclass[hidelinks,onefignum,onetabnum]{siamart220329}

\usepackage{lipsum}
\usepackage{amsfonts}
\usepackage{graphicx}
\usepackage{epstopdf}
\usepackage{algorithmic}
\ifpdf
  \DeclareGraphicsExtensions{.eps,.pdf,.png,.jpg}
\else
  \DeclareGraphicsExtensions{.eps}
\fi
\usepackage[export]{adjustbox}
\usepackage{amssymb}
\usepackage{cite}
\usepackage{color}
\usepackage{booktabs}
\usepackage{float}
\usepackage{mathtools}
\usepackage{multirow}
\usepackage{siunitx}
\usepackage{bm}
\usepackage{array}
\usepackage[caption=false]{subfig}
\usepackage{layout}

\usepackage{enumitem}

\newcommand{\revision}[1]{{\color{black} #1}}

\newcommand{\centerhfill}[1][\quad]{\hspace{\stretch{0.5}}#1\hspace{\stretch{0.5}}}

\newcommand{\multilinecell}[3][c]
{
\begin{tabular}{@{}#1@{}}#2 \\ #3\end{tabular}
}

\newsiamremark{remark}{Remark}
\newsiamremark{hypothesis}{Hypothesis}
\crefname{hypothesis}{Hypothesis}{Hypotheses}
\newsiamthm{claim}{Claim}
\newsiamthm{property}{Property}
\crefname{property}{Property}{Properties}
\newsiamthm{assumption}{Assumption}
\crefname{assumption}{Assumption}{Assumptions}

\newcommand{\crefbookmark}[2]{\texorpdfstring{\cref{#2}}{#1~\ref*{#2}}}

\headers{
CBFs for Collision Avoidance between Convex Regions
}{A. Thirugnanam, J. Zeng and K. Sreenath}

\title{
Control Barrier Functions for Collision Avoidance between Strongly Convex Regions
\thanks{
\textbf{Funding:} This work was partially supported through funding from the Tsinghua-Berkeley Shenzhen Institute (TBSI) program.
}
}

\author{
Akshay Thirugnanam\thanks{Department of Mechanical Engineering, UC Berkeley, CA, USA 
  (\email{akshay\_t@berkeley.edu}, \email{zengjunsjtu@berkeley.edu}, \email{koushils@berkeley.edu}).}
\and Jun Zeng\footnotemark[2]
\and Koushil Sreenath\footnotemark[2]
}

\usepackage{amsopn}

\DeclareMathOperator*{\argmin}{argmin} 

\begin{document}


\maketitle

\begin{abstract}
In this paper, we focus on non-conservative \revision{collision} avoidance between robots and obstacles with control affine dynamics and convex shapes.
System safety is defined using the minimum distance between the safe regions associated with robots and obstacles.
However, \revision{collision} avoidance using the minimum distance as a control barrier function (CBF) can pose challenges because the minimum distance is implicitly defined by an optimization problem and thus nonsmooth in general.
\revision{
We identify a class of state-dependent convex sets, defined as strongly convex maps, for which the minimum distance is continuously differentiable, and the distance derivative can be computed using KKT solutions of the minimum distance problem.
In particular, our formulation allows for ellipsoid-polytope collision avoidance and convex set algebraic operations on strongly convex maps.
We show that the KKT solutions for strongly convex maps can be rapidly and accurately updated along state trajectories using a KKT solution ODE.
Lastly, we propose a QP incorporating the CBF constraints and prove strong safety under minimal assumptions on the QP structure.
We validate our approach in simulation on a quadrotor system navigating through an obstacle-filled corridor and demonstrate that CBF constraints can be enforced in real time for state-dependent convex sets without overapproximations.
}

\end{abstract}



\section{Introduction}
\label{sec:introduction}

%
%
%

Safety-critical control and planning methods typically ensure collision avoidance by enforcing the minimum distance between robots and obstacles to be positive along the state trajectory.
For continuous-time safety-critical methods using control barrier functions (CBFs), collision avoidance can be enforced by choosing the CBF as the minimum distance between robots and obstacles.
Safety for CBF-based methods is guaranteed by enforcing the CBF constraint, which depends on the gradient (or generalized gradient) of the CBF.
However, when the shapes of the robots and obstacles are described by general convex sets, the minimum distance is implicitly defined as the solution to an optimization problem and may not be differentiable.
CBF-based collision avoidance methods often use overapproximations of robots and obstacles to obtain explicit, differentiable minimum distance functions.
\revision{However, such overapproximations can result in deadlocks and can be challenging to compute in real time when the robot or obstacle shapes are state-dependent.}
In this work, we consider a large class of state-dependent convex sets, defined by \emph{strongly convex maps}, for which the minimum distance function is continuously differentiable.
\revision{
We show that the distance derivatives can be computed using the KKT solutions of the minimum distance problem and that the KKT solutions can be quickly propagated along state trajectories.
Our proposed method demonstrates that for strongly convex maps, CBF constraints can be enforced in real time (at $\SI{500}{Hz}$) without overapproximations and with minimal computational penalties.
The code for the examples in the paper can be found in the repository\footnote{
\label{code}
\revision{A C\texttt{++} library implementing the proposed algorithm and the code for the examples in the paper can be found at \url{https://github.com/HybridRobotics/cbf-convex-maps}.}
}.
}

\subsection{Related work and contributions}
\label{subsec:related-work-and-contributions}

\subsubsection{Collision avoidance between convex sets}
\label{subsec:collision-avoidance-between-convex-sets}

\revision{
Safety in trajectory optimization and motion planning problems is often achieved by enforcing collision avoidance constraints using the minimum distance between robots and obstacles.
Collision avoidance can be enforced as hard or soft constraints (using a collision penalty cost term).
The minimum distance between convex sets is defined implicitly using an optimization problem, and thus, efficient methods to compute the distance and its gradient are required for solving safety-critical optimization problems.

The first approach to solving distance-based optimization problems is to compute the minimum distance explicitly and numerically approximate the gradient.
Methods to compute the distance (and the distance gradient) in this approach include signed distance fields (SDFs) \cite{ratliff2009chomp}, the GJK algorithm for convex sets \cite{gilbert1988fast,schulman2013finding,schulman2014motion}, and the growth distance metric \cite{ong1996growth,ong1994robot,tracy2023differentiable}.

The second approach is to overapproximate the shapes of robots and obstacles using convex sets for which the distance (and its gradient) can be easily computed.
When the robots and obstacles are overapproximated by lines, planes, circles, ellipses, or spheres~\cite{kalakrishnan2011stomp, chen2019autonomous, singh2021optimizing}, the distance (or equivalent metrics) can be easily obtained.
While efficient methods to compute overapproximations have been proposed \cite{zhu1999new,todd2007khachiyan}, such methods cannot be used in real-time implementations for state-dependent robot shapes because the overapproximations have to be recomputed at each state.
Another method, dual to robot and obstacle shape overapproximation, is free space underapproximation.
In such methods, the obstacle-free space is underapproximated by ellipsoids or polytopes \cite{barbosa2020provably,deits2015computing,yan2015closed}.

The third approach is to convert implicitly-defined distance-based safety constraints to explicit constraints.
When the obstacle shapes can be described by polyhedra, mixed-integer programming can be used to enforce safety constraints by partitioning the obstacle-free space \cite{schouwenaars2001mixed, grossmann2002review, richards2005mixed}.
However, such methods cannot be used for real-time implementations.
Convex duality theory can be used to reformulate minimum distance constraints as maximum separation constraints \cite{zhang2018autonomous,zeng2020differential,firoozi2021formation,thirugnanam2022safetycritical}.
While computationally faster than mixed-integer formulations, real-time applications are limited to linear systems or short-horizon MPC use cases.
}

\subsubsection{Collision avoidance using control barrier functions}
\label{subsec:collision-avoidance-using-control-barrier-functions}

\revision{
Similar to control Lyapunov functions (CLFs) for stabilization, control barrier functions (CBFs) are used to encode safe sets and enforce continuous-time safety constraints for dynamical systems \cite{ames2019control}.
CBF constraints can be used to design CBF-QP-based safety filters that guarantee safety by minimally perturbing a given reference control input \cite{ames2014rapidly,ames2016control}.
The CBF-QP formulation enables safe real-time operation for dynamical systems due to the low computational complexity of the CBF-QP.
However, accurate gradient (or generalized gradient) information is required to enforce CBF constraints and thus theoretically guarantee safety using CBFs.
This can pose challenges for the methods described in the first approach in \cref{subsec:collision-avoidance-between-convex-sets}, even if the gradient is explicitly computable and accurate almost everywhere on the state space \cite[Ex.~16]{cortes2008discontinuous}.

One method to obtain an explicit, differentiable distance function is to overapproximate robot and obstacle shapes.
Previous works have considered point-masses \cite{chen2017obstacle}, circles \cite{reis2020control}, parabolas \cite{ferraguti2020control} and high-dimensional spheres \cite{wu2015safety}.
Collision avoidance approaches for ellipsoids and conic sections have also been proposed, including separating plane-based CBFs for ellipsoidal agents \cite{nishimoto2022collision,funada2025collision}, collision cones for quadric functions \cite{best2016real,kashyap20243d,dhal2023robust}, and determinant-based closed-form CBFs for ellipsoids \cite{choi2006continuous,verginis2019closed}.
Our previous work in \cite{thirugnanam2022duality} discussed collision avoidance for polytopes using CBFs.
We note that previous works on CBFs have not considered ellipsoid-polytope collision avoidance, which is commonly required when robot links are defined by ellipsoids and obstacles by polytopes.

The \emph{first main contribution} of our paper is to show that for a large class of convex sets, defined by \emph{smooth} and \emph{strongly convex maps}, the minimum distance is continuously differentiable and the derivative can be computed using KKT solutions.
Our formulation includes ellipsoid-polytope collision avoidance and state-dependent convex safe regions and allows for algebraic operations (such as Minkowski sums, Cartesian products, intersections, and projections) on strongly convex maps.
Further, we show that the KKT solutions can be propagated quickly along a state trajectory using a KKT solution ODE.
Our simulations demonstrate that the KKT solutions can be accurately updated in real time (on the order of $10 \mu s$).
Combining the above two results, we propose a CBF-QP with the same sparsity structure as the explicit distance case (note that this is not the case in \cite{thirugnanam2022duality}).
Thus, we claim that for strongly convex maps, collision avoidance using CBFs can be directly implemented without overapproximations and with minimal computational penalties, allowing for real-time collision avoidance with complex safe regions and fewer deadlocks.
}

\subsubsection{CBFs for discontinuous dynamical systems}
\label{subsec:nonsmooth-cbfs-for-discontinuous-dynamical-systems}

\revision{
The proof of safety for CBFs assumes a locally Lipschitz continuous feedback control law that satisfies the CBF constraint \cite{ames2019control}.
Since CBFs are generally used as safety filters in the form of CBF-QPs \cite{ames2014rapidly,ames2016control}, many previous works prove safety by considering the local Lipschitz continuity of the CBF-QP optimal solution.
The Lipschitz continuity of the feedback control law can be used to prove the uniqueness of closed-loop state trajectories.
However, guaranteeing the Lipschitz continuity of the CBF-QP optimal solution requires many assumptions on the CBF-QP structure, which may not be guaranteed in practice \cite{morris2013sufficient,morris2015continuity}.
Further, such assumptions require that the CBF-QP be solved to optimality, i.e., when the CBF-QP solution is suboptimal, the Lipschitz continuity property may not hold.
While previous works on nonsmooth CBFs have extended the existing safety results to nonsmooth functions and discontinuous dynamical systems \cite{glotfelter2017nonsmooth,glotfelter2018boolean,glotfelter2019hybrid,glotfelter2021nonsmooth}, the properties of the CBF-QP are not extensively discussed.

The \emph{second main contribution} of our paper is to show that safety for the closed-loop system can be guaranteed under relaxed assumptions on the CBF-QP structure.
We relax the assumptions on the CBF-QP by forgoing the uniqueness property of the closed-loop state trajectory and proving \emph{strong safety}, i.e., that the closed-loop system remains safe for all state trajectories.
Our proof of safety requires minimal assumptions on the CBF-QP structure (continuity of the CBF-QP constraints) and is also valid for suboptimal (but feasible) optimal solutions and arbitrary costs.
}







\subsection{Notation}
\label{subsec:notation}

For $n \in \mathbb{N}$, $[n]$ denotes the set $\{1, 2, ..., n\}$.
Superscripts of variables, such as $x^i$, denote the robot/set index.
Subscripts of variables, such as $A_{k}$, denote the row index of vectors or matrices, and stylized subscripts, such as $A_{\mathcal{R}}$, denote the sub-matrix obtained by selecting those rows whose index lies in the set $\mathcal{R}$.
For proofs, subscripts of variables with parenthesis, such as $a_{(m)}$, denote the $m$-th element in the sequence $\{a_{(m)}\}$.
\Cref{tab:symbols-notations} summarizes the symbols used in the paper.

\subsection{Paper structure}
The paper is organized as follows.
\revision{\Cref{sec:problem-description-outline-main-results} describes the problem statement and summarizes the main results of the paper.}
A brief introduction to discontinuous dynamical systems and nonsmooth CBFs is presented in \cref{sec:background}.
We analyze the continuity and differentiability properties of the minimum distance problem between strongly convex maps in \cref{sec:minimum-distance-between-strongly-convex-maps-smoothness-properties} and propose a CBF-QP-based obstacle avoidance formulation for strongly convex maps in \cref{sec:ncbfs-for-strictly-convex-sets}.
We validate our approach in simulations in \cref{sec:results} and present concluding remarks in \cref{sec:conclusion}.

\section{Problem description and outline of main results}
\label{sec:problem-description-outline-main-results}

\revision{\phantom{}}
We consider the enforcement of safety constraints between multiple controllable robots and obstacles using control barrier functions (CBFs).
Each controlled robot is associated with some states, control inputs, system dynamics, which describe the state evolution, and geometries, which represent the \revision{safe regions associated with each robot}.
\revision{These safe regions can include the physical space occupied by the robots, uncertainty bounds, and other safety measures needed to guarantee the safe operation of the system (for examples, see \cref{sec:results}).}
Enforcing safety for such systems requires their feedback control laws to guarantee that no two robots or obstacles \revision{have their respective safe regions collide}, i.e., the minimum distance between \revision{the safe regions of} any two robots \revision{(or robot-obstacle pair)} is always greater than zero.

\subsection{Problem description}
\label{subsec:problem-description}

\revision{\phantom{}}
Consider $N$ robots with the Robot $i$ having states $x^i \in \mathcal{X}^i \subset \mathbb{R}^n$ and nonlinear, control-affine dynamics:
\begin{equation}
\label{eq:n-affine-systems}
\dot{x}^i(t) = f^i(x^i(t)) + g^i(x^i(t))u^i(t), \quad i \in [N],
\end{equation}
where $f^i: \mathcal{X}^i \rightarrow \mathbb{R}^n$, $g^i: \mathcal{X}^i \rightarrow \mathbb{R}^{n\times m}$, and $u^i(t) \in \mathcal{U}^i \subset \mathbb{R}^m$.
\revision{
For the rest of the paper, we assume that for all $i \in [N]$, $f^i$ and $g^i$ are continuous functions, $\mathcal{X}^i$ is an open connected set, and $\mathcal{U}^i$ a convex compact set.
For the convenience of notation, we represent obstacles as robots with no control inputs, i.e., with $m = 0$.
Thus, we only consider safety between robots for the theoretical development.
}

\revision{The state-dependent safe region associated with Robot $i$ is described by the set-valued map $\mathcal{C}^i: \mathcal{X}^i \rightarrow 2^{\mathbb{R}^l}$, where $\mathcal{C}^i(x^i)$ is the safe region for Robot $i$ at state $x^i$.}
Here $2^{\mathbb{R}^l}$ denotes the power set of $\mathbb{R}^l$.
We assume that $\mathcal{C}^i$ has the following form:
\begin{equation}
\label{eq:lipschitz-convex-set}
\mathcal{C}^i(x^i) = \{z^i \in \mathbb{R}^l: A^i_k(x^i, z^i) \leq 0, \ \forall k \in [r^i]\},
\end{equation}
where $A^i_k: \mathcal{X}^i \times \mathbb{R}^l \rightarrow \mathbb{R}$, and $r^i$ is the number of constraints used to define $\mathcal{C}^i(x^i)$.
\revision{Let $A^i: \mathcal{X}^i \times \mathbb{R}^l \rightarrow \mathbb{R}^{r^i}$ be defined as $A^i(x^i, z^i) = (A^i_1(x^i, z^i), ..., A^i_{r^i}(x^i, z^i))$.}
\revision{We assume that the set-valued safe region maps $\mathcal{C}^i$ satisfy certain smoothness properties, which will allow us to enforce CBF constraints to guarantee safety.
In particular, we will assume that the set-valued maps $\mathcal{C}^i$ are \emph{strongly convex maps}, which are central to the theory developed in this paper.}

\revision{To define strongly convex maps, we first define the set of active indices.}
For a given state $x^i$ and a point $z^i \in \mathcal{C}^i(x^i)$, the set of active indices is defined as,
\begin{equation}
\label{eq:strict-convex-index-set-j}
\mathcal{J}^i(x^i, z^i) := \{k \in [r^i]: A^i_k(x^i, z^i) = 0\}.
\end{equation}
%
Now, we define \emph{smooth convex maps} and strongly convex maps. 
\begin{definition}[Smooth convex map]%
\label{def:smooth-convex-set}%
A set-valued map $\mathcal{C}^i$ of the form \cref{eq:lipschitz-convex-set} is called a smooth convex map if:
\begin{enumerate}[ref={\thedefinition.\arabic*}, leftmargin=*]
    \item \label[definition]{subdef:smooth-convex-set-c2}
    $A^i_k$ is twice continuously differentiable on $\mathcal{X}^i \times \mathbb{R}^l$ $\forall k \in [r^i]$.
    
    \item \label[definition]{subdef:smooth-convex-set-licq}
    The set $\mathcal{C}^i(x^i)$ satisfies linear independence constraint qualification (LICQ), i.e., the set of gradients of active constraints, $\{\nabla_{z^i} A^i_k(x^i,z^i): k\in \mathcal{J}^i(x^i, z^i)\}$ is linearly independent for all $z^i \in \mathcal{C}^i(x^i)$ and $x^i \in \mathcal{X}^i$.
    
    \item \label[definition]{subdef:smooth-convex-set-slater}
    For all $x^i \in \mathcal{X}^i$, $\mathcal{C}^i(x^i)$ is a compact set and has a non-empty interior.
\end{enumerate}

\end{definition}

\begin{definition}[Strongly convex map]%
\label{def:strongly-convex-set}%
A smooth convex map $\mathcal{C}^i$ is called a strongly convex map \revision{if $\nabla_{z^i}^2A^i_k(x^i, z^i) \succ 0$, $\forall k \in [r^i]$, $z^i \in \mathbb{R}^l$, and $x^i \in \mathcal{X}^i$.}

\end{definition}
\revision{We note that since $A^i_k$ is twice continuously differentiable and $\mathcal{C}^i(x^i)$ is compact (by \cref{subdef:smooth-convex-set-slater}), $\nabla_{z^i}^2A^i_k(x^i, \cdot)$ is uniformly positive definite in any compact neighborhood of $\mathcal{C}^i(x^i)$; this justifies the term strongly convex map.}
\revision{
For a pair of robots $(i,j)$ for which collision avoidance is to be enforced, we call the tuple $(\mathcal{C}^i, \mathcal{C}^j)$ a \emph{collision pair}.
\Cref{fig:quadrotor-safe-set} shows two strongly convex maps for a quadrotor system.
}
The following assumption is used throughout the rest of the paper.

\begin{assumption}[Strongly convex pair]%
\label{assum:strongly-convex-pair}%
For all $i \in [N]$, $\mathcal{C}^i$ is a smooth convex map.
Further, for all $i, j \in [N]$ such that $(\mathcal{C}^i, \mathcal{C}^j)$ is a collision pair, at least one is a strongly convex map.

\end{assumption}

\revision{
For a collision pair $(\mathcal{C}^i, \mathcal{C}^j)$ and states $x^i \in \mathcal{X}^i$, $x^j \in \mathcal{X}^j$, (the square of) }
the minimum distance $h^{ij}(x^i, x^j)$ between $\mathcal{C}^i(x^i)$ and $\mathcal{C}^j(x^j)$ is given by the optimization problem,
\begin{equation}
\label{eq:convex-set-min-dist}
\begin{split}
h^{ij}(x^i, x^j) := & \ \min_{(z^i, z^j)} \ \ \{\lVert z^i - z^j \rVert_2^2: \ z^i \in \mathcal{C}^i(x^i), \ z^j \in \mathcal{C}^j(x^j)\}, \\
\revision{=} & \ \revision{\min_{(z^i, z^j)} \ \ \{\lVert z^i - z^j \rVert_2^2: \ A^i(x^i, z^i) \leq 0, A^j(x^j, z^j) \leq 0\}.} 
\end{split}
\end{equation}
\revision{Safety for the collision pair $(\mathcal{C}^i, \mathcal{C}^j)$ is expressed by the condition $h^{ij}(x^i, x^j) > 0$ for the entire state trajectory (safety will be concretely defined in \cref{subsec:nonsmooth-control-barrier-functions}).}

\begin{figure}[tbp]
    \centering
    \subfloat[
    \revision{Robust safe region for a quadrotor, including quadrotor shape and position uncertainty.}
    ]{
    \label{subfig:quadrotor-uncertainty-safe-set}
    {
    \includegraphics[width=4.8in]{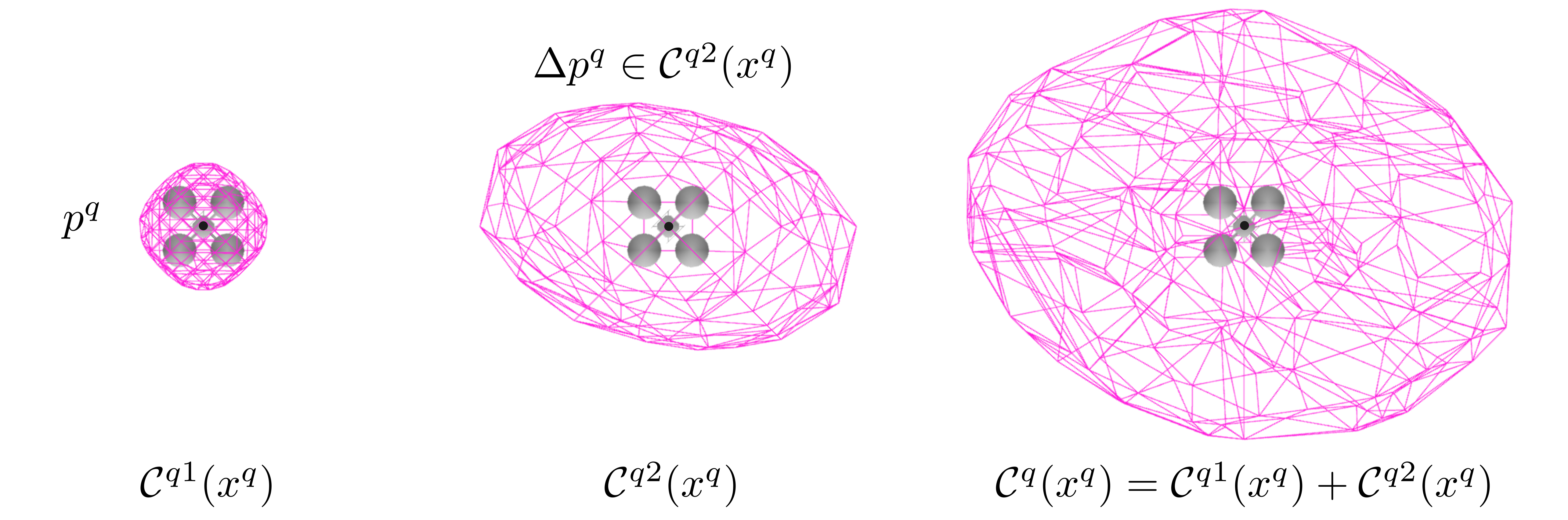} 
    }%
    }%
    
    \subfloat[
    \revision{Dynamic safe region for a quadrotor, including quadrotor shape and braking corridor.}
    ]{
    \label{subfig:quadrotor-corridor-safe-set}
    {
    \includegraphics[width=4.8in]{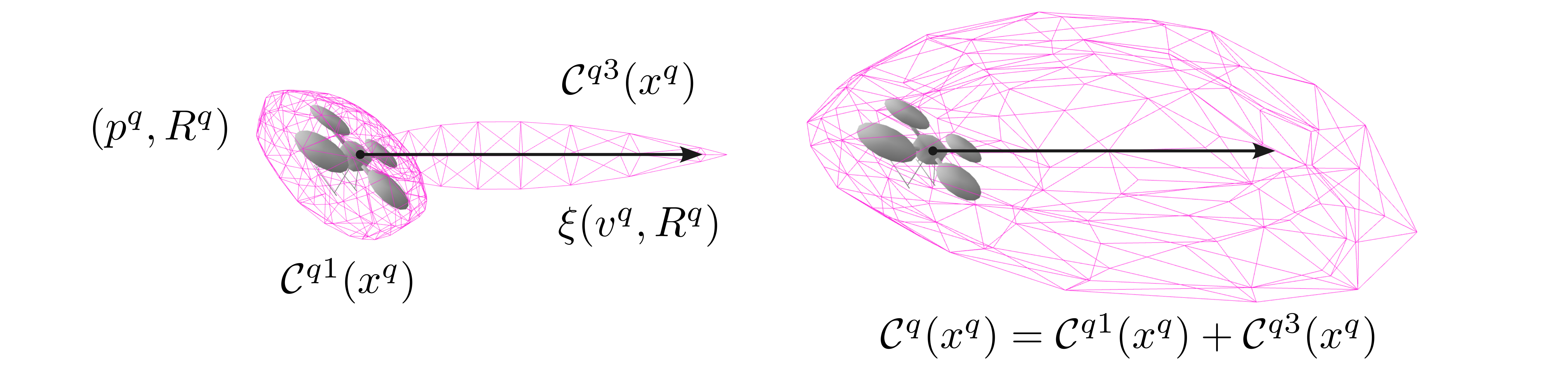}
    }%
    }%
    \caption{
    \revision{
    Examples of safe regions for a quadrotor system with state $x^q = (p^q, R^q, v^q) \in SE(3) \times \mathbb{R}^3$.
    In \cref{subfig:quadrotor-uncertainty-safe-set,subfig:quadrotor-corridor-safe-set}, the set $\mathcal{C}^{q1}(x^q)$ represents the shape of the quadrotor at the state $x^q$.
    \Cref{subfig:quadrotor-uncertainty-safe-set} depicts a safe region that considers the position uncertainty (represented by $\mathcal{C}^{q2}(x^q)$) of the quadrotor.
    The robust safe region $\mathcal{C}^{q}(x^q)$ is constructed using the Minkowski sum of the quadrotor shape $\mathcal{C}^{q1}(x^q)$ and uncertainty set $\mathcal{C}^{q2}(x^q)$.
    \Cref{subfig:quadrotor-corridor-safe-set} depicts a safe region that expands the quadrotor shape set along a braking corridor $\mathcal{C}^{q3}(x^q)$ (see \cref{subsec:example-cbf-obstacle avoidance}).
    The braking distance vector is given by $\xi(v^q, R^q)$ and parameterized by the velocity and orientation of the quadrotor.
    The dynamic safe region $\mathcal{C}^{q}(x^q)$ is constructed using the Minkowski sum of the quadrotor shape $\mathcal{C}^{q1}(x^q)$ and braking corridor $\mathcal{C}^{q3}(x^q)$.
    The safe regions for both examples can be represented by strongly convex maps (see \cref{def:strongly-convex-set}) and are used for the results in \cref{sec:results}.
    }
    }
    \label{fig:quadrotor-safe-set}
\end{figure}

\subsection{Outline of the main results}
\label{subsec:outline-of-the-main-results}

\revision{\phantom{}}
\revision{
In this subsection, we outline the three main results in this paper. 
All results are stated for a collision pair $(\mathcal{C}^i, \mathcal{C}^j)$.
First, we state a preliminary result on the KKT solutions of \cref{eq:convex-set-min-dist}.

\emph{
(Derivatives of the KKT solution, \cref{lem:strict-convex-unique-continuous-kkt-solution}, \cref{prop:strict-convex-kkt-solution-directional-derivative})
Let \cref{assum:strongly-convex-pair} hold, and consider $x = (x^i, x^j) \in \mathcal{X}^i \times \mathcal{X}^j$ such that $h^{ij}(x) > 0$.
Then, there is a unique, continuous KKT solution $(z^*(x'), \lambda^*(x'))$ for $x'$ in a neighborhood of $x$.
Further, $(z^*(\cdot), \lambda^*(\cdot))$ is directionally differentiable at $x$.
}
}

\revision{
Next, we summarize the three main results in the paper.
The first result allows us to propagate KKT solutions along differentiable state trajectories quickly.
\begin{enumerate}[labelindent=0pt, itemindent=!, leftmargin=*]
    \item \emph{
    (KKT solution ODE, \cref{thm:distance-ode})
    Let \cref{assum:strongly-convex-pair} hold, and consider a differentiable state trajectory $x(t) = (x^i(t)$, $x^j(t))$ such that $h^{ij}(x(t)) > 0$ $\forall t \geq t_0$.
    Then, the KKT solution $(z^*(x(t)), \lambda^*(x(t)))$ is differentiable with respect to $t$ and can be obtained as a solution of an ODE starting from $(z^*(x(t_0)), \lambda^*(x(t_0)))$.
    }
\end{enumerate}
\vspace{5pt}
}

\revision{
The second result discusses the smoothness properties of the minimum distance function and computes the minimum distance derivatives using the KKT solution.
The derivative of $h^{ij}$ will be used to enforce the CBF constraint.
\begin{enumerate}[labelindent=0pt, itemindent=!, leftmargin=*]
    \setcounter{enumi}{1}
    \item \emph{
    (Derivative of minimum distance, \cref{thm:minimum-distance-derivative})
    Let \cref{assum:strongly-convex-pair} hold, and consider $x = (x^i, x^j)$ such that $h^{ij}(x) > 0$.
    Then, the minimum distance function $h^{ij}$ 
    is continuously differentiable in a neighborhood of $x$, with
    \vspace{-6pt}
    \begin{equation*}
        D_x h^{ij}(x) = \bigl(\lambda^{i*}(x)^\top D_{x^i} A^i(x^i, z^{i*}(x)), \, \lambda^{j*}(x)^\top D_{x^j} A^j(x^j, z^{j*}(x)) \bigr).
    \end{equation*}
    }
\end{enumerate}
\vspace{2pt}
}

\revision{
The final result provides a CBF-QP-based feedback control law that guarantees the safety of the closed-loop system.
In particular, even if the feedback control law is not locally Lipschitz (and closed-loop state trajectories are not unique), the CBF-QP guarantees \emph{strong safety}, i.e., all closed-loop trajectories are safe.
Thus, unlike many previous works, we prove safety without relying on the Lipschitz continuity of the optimal solution of the CBF-QP.
\begin{enumerate}[labelindent=0pt, itemindent=!, leftmargin=*]
    \setcounter{enumi}{2}
    \item \emph{
    (CBF-QP for strongly convex pairs, \cref{thm:cbf-qp})
    Let \cref{assum:strongly-convex-pair} hold and $x_0 = (x_0^i, x_0^j)$ be such that $h^{ij}(x) > 0$.
    Then, any measurable feedback control law $u^*_{fb}: \mathcal{X}^i \times \mathcal{X}^j \rightarrow \mathcal{U}^i \times \mathcal{U}^j$ that is feasible for the CBF-QP \cref{eq:cbf-qp} guarantees strong safety for the closed-loop system.
    }
\end{enumerate}
}

\revision{Our formulation also allows for algebraic operations on the strongly convex maps, including projections, Cartesian products, Mikowski sums, and intersections.}

\revision{
In \cref{sec:minimum-distance-between-strongly-convex-maps-smoothness-properties,sec:ncbfs-for-strictly-convex-sets}, we discuss the smoothness properties of $h^{ij}$ and provide proofs for the main results of the paper.
We start by presenting some background on differential inclusions and nonsmooth CBFs in \cref{sec:background}.
}

\section{Background}
\label{sec:background}


\revision{
Consider the dynamical system in \cref{eq:n-affine-systems} with the inputs given by a state feedback control law as $u^i(t) = u^i_{fb}(x(t))$, where $x(t)$ is the state of the full system at time $t$.
Then, the closed-loop system dynamics can be written as
\begin{equation}
\label{eq:n-affine-systems-feedback}
\dot{x}^i(t) = f^i(x^i(t)) + g^i(x^i(t))u^i_{fb}(x(t)), \quad i \in [N].
\end{equation}
If the functions $f^i$, $g^i$, and $u^i_{fb}$ are locally Lipschitz continuous in $x$, then the closed-loop system has a unique local solution (see~\cite[Thm.~54]{sontag2013mathematical} for weaker conditions).
However, when the feedback control law $u^i_{fb}$ is computed using a CBF-QP, guaranteeing Lipschitz continuity of $u^i_{fb}$ requires assumptions on the parametric form of the QP that may not hold in practice~\cite{morris2013sufficient,morris2015continuity}.
Further, such conditions only guarantee Lipschitz continuity when the CBF-QP is solved to optimality.
For an example of a strictly convex parametric QP with non-Lipschitz optimal solution, see \cite{robinson1982generalized}.
}

\revision{
When $u^i_{fb}$ is only continuous (or discontinuous), the closed-loop system may not have a unique solution.
Thus, we forgo the Lipschitz continuity and continuity properties of $u^i_{fb}$ and prove \emph{strong safety} for the closed-loop system, i.e., all closed-loop system trajectories are safe.
As we will show in \cref{thm:cbf-qp}, this approach has the desirable property that any measurable feedback control law $u^i_{fb}$, such that $u^i_{fb}(x)$ is feasible for the CBF-QP for all $x$, guarantees strong safety.
With this motivation, we first present some background on Differential inclusions and Filippov solutions.
}

\subsection{Discontinuous dynamical systems}
\label{subsec:discontinuous-dynamical-systems}

To have a well-defined notion of a solution to an ODE with a discontinuous RHS, we can study the properties of differential inclusions of the form:
\begin{equation}
\label{eq:differential-inclusion}
\dot{x}^i(t) \in F^i(x(t)), \quad i \in [N], \quad x(t_0) = x_0,
\end{equation}
where $F^i:\mathcal{X} \rightarrow 2^{\mathbb{R}^n}$ is a set-valued map.
We also need the notion of continuity for set-valued maps.

\begin{definition}[Semi-continuity]~\cite[Sidebar~7]{cortes2008discontinuous}%
\label{def:usc-lsc}
A set-valued map $\Gamma: \mathcal{X} \rightarrow 2^{\mathbb{R}^n}$ is upper semi-continuous (respectively, lower semi-continuous) at $a \in \mathcal{X}$, if $\forall \epsilon > 0$, $\exists \delta > 0$ such that $\forall x \in \mathcal{B}_\delta(a)$, $\Gamma(x) \subset \Gamma(a) + \mathcal{B}_\epsilon(0)$ (respectively, $\Gamma(a) \subset \Gamma(x) + \mathcal{B}_\epsilon(0)$).
Here $\mathcal{B}_r(x) := \{y: \lVert x-y\rVert < r\}$ is an open ball of radius $r$ around $x$, and the addition between sets is defined as Minkowski addition~\cite{gritzmann1993minkowski}.
For sets $A$ and $B$, the Minkowski addition is defined as $A+B := \{a+b: a \in A, b \in B\}$.
$\Gamma$ is continuous at $a \in \mathcal{X}$ if it is both upper and lower semi-continuous at $a$.

\end{definition}

For a differential inclusion of the form of \cref{eq:differential-inclusion}, a solution can be defined as follows.

\begin{definition}[Caratheodory solution]~\cite{cortes2008discontinuous}
\label{def:caratheodory-solution}
A Caratheodory solution to \cref{eq:differential-inclusion} on $[t_0, T]$ is an absolutely continuous map $x^i: [t_0, T] \rightarrow \mathcal{X}^i$, $i \in [N]$, satisfying \cref{eq:differential-inclusion} for almost all $t \in [t_0, T]$, i.e., the set of times when \cref{eq:differential-inclusion} is not satisfied has measure zero.

\end{definition}

We now examine the dynamical system \revision{\cref{eq:n-affine-systems-feedback}} with a discontinuous feedback control law \revision{$u^i_{fb}(x)$} and convert it into the form \cref{eq:differential-inclusion} to obtain a Caratheodory solution for the system.
To do this, we use the Filippov operator.
For a vector field $f^i: \mathcal{X} \rightarrow \mathbb{R}^n$, the Filippov operator on $f^i$, $F[f^i]: \mathcal{X} \rightarrow 2^{\mathbb{R}^n}$, is defined as \cite[Eq.~(19,30)]{cortes2008discontinuous},
\begin{equation}
\label{eq:filippov-operator-def-limit}
F[f^i](x) := \bigcap_{\mu(\mathcal{Q}) = 0}\text{cl } \text{conv}\bigl\{\lim_{k\rightarrow \infty} f^i(x_{(k)}): x_{(k)} {\rightarrow} x, \; x_{(k)} {\notin} \mathcal{Q}_d \cup \mathcal{Q}\bigr\},
\end{equation}
where `cl' denotes closure, `conv' denotes convex hull, $\mathcal{Q}_d$ is the set of the points of discontinuities of $f^i$, which is measure-zero, and $\mu$ denotes Lebesgue measure.

Let, \revision{for all $i \in [N]$, $u^i_{fb}: \mathcal{X} \rightarrow \mathcal{U}^i$} be some measurable feedback control law.
The discontinuous dynamical system \revision{\cref{eq:n-affine-systems-feedback}} can be converted to a differential inclusion of the form \cref{eq:differential-inclusion}, using the Filippov operator $F[\cdot]$ as:
\begin{equation}
\label{eq:filippov-operator-def}
\dot{x}^i(t) \in \revision{F[f^i+g^iu^i_{fb}](x(t)), \quad i \in [N]}.
\end{equation}
We can now define a solution of \revision{\cref{eq:n-affine-systems-feedback}} as a solution of \cref{eq:filippov-operator-def}.

\begin{definition}[Filippov solution]~\cite[Eq.~(21)]{cortes2008discontinuous}
\label{def:filippov-solution}
A Filippov solution of \revision{\cref{eq:n-affine-systems-feedback}} is a Caratheodory solution of \cref{eq:filippov-operator-def}.

\end{definition}

Caratheodory solutions are defined for differential inclusions of the form \cref{eq:differential-inclusion}, whereas Filippov solutions are defined for ODEs of the form \revision{\cref{eq:n-affine-systems-feedback}} by converting the ODE into a differential inclusion using \cref{eq:filippov-operator-def}.
The following result guarantees the existence of Filippov solutions.

\begin{proposition}[Existence of Filippov solution]~\cite[Prop.~3]{cortes2008discontinuous}
\label{prop:existence-filippov-solution}
The map \revision{$F[f^i+g^iu^i_{fb}]:\mathcal{X} \rightarrow 2^{\mathbb{R}^n}$} is upper semi-continuous and is non-empty, convex, and compact at each $x \in \mathcal{X}$.
For all $x_0 \in \mathcal{X}$, \cref{eq:filippov-operator-def} has a Caratheodory solution with $x(t_0) = x_0$, which is a Filippov solution to \revision{\cref{eq:n-affine-systems-feedback}}.

\end{proposition}

\subsection{Nonsmooth control barrier functions}
\label{subsec:nonsmooth-control-barrier-functions}

\phantom{}
\revision{For the rest of this section and \cref{sec:minimum-distance-between-strongly-convex-maps-smoothness-properties,sec:ncbfs-for-strictly-convex-sets}, we will restrict our discussion to a collision pair $(\mathcal{C}^i, \mathcal{C}^j)$ between Robots $i$ and $j$ for simplicity.}
We also define $\mathcal{X} := \mathcal{X}^i \times \mathcal{X}^j$, $\mathcal{U} := \mathcal{U}^i \times \mathcal{U}^j$, $x := (x^i,x^j) \in \mathcal{X}$, $h(x) := h^{ij}(x^i,x^j)$, $z := (z^i, z^j)$, and $u := (u^i,u^j) \in \mathcal{U}$. 

To enforce \revision{collision avoidance between the safe regions of} Robots $i$ and $j$, we want the minimum distance $h(x)$ (see \cref{eq:convex-set-min-dist}) between \revision{$\mathcal{C}^i(x^i)$ and $\mathcal{C}^j(x^j)$} to be greater than $0$.
We can define the set of safe states~\cite{ames2019control} as 
\begin{equation}
\label{eq:safe-set-def}
\mathcal{S} := \text{cl} \{ x: h(x) > 0 \}.
\end{equation}

\begin{remark}[Definition of the safe set]
We define $\mathcal{S}$ as the closure of the actual safe set to ensure that $\mathcal{S}$ is a closed set, but this may also introduce states that are not safe into $\mathcal{S}$.
Under certain regularity assumptions on the \revision{safe regions $\mathcal{C}^i$ and $\mathcal{C}^j$} of Robots $i$ and $j$, taking the closure of the safe set introduces only those unsafe states $x$ in which \revision{$\mathcal{C}^i(x^i) \cap \mathcal{C}^j(x^j)$ has measure zero}~\cite{rodriguez2012path}.
Thus, it is not detrimental to the problem of \revision{collision} avoidance.

\end{remark}

\revision{To achieve collision avoidance for the closed-loop system, we first define nonsmooth control barrier functions (adapted from ~\cite[Def.~4]{glotfelter2017nonsmooth}) and strong safety.}

\begin{definition}[Nonsmooth control barrier function \revision{and strong safety}]~\cite[Def.~4]{glotfelter2017nonsmooth}
\label{def:ncbf}
Let $h: \mathcal{X} \rightarrow \mathbb{R}$ be a locally Lipschitz continuous function and \revision{$\mathcal{S} = \text{cl} \, \{x \in \mathcal{X}: h(x) > 0\}$}.
Then, $h$ is a nonsmooth control barrier function (NCBF) if there exist \revision{measurable feedback control laws $u^i_{fb}: \mathcal{X} \rightarrow \mathcal{U}^i$ and $u^j_{fb}: \mathcal{X} \rightarrow \mathcal{U}^j$ such that for all $x_0$ with $h(x_0) > 0$, $x(t) \in \mathcal{S}$ for all $t \in [t_0, T]$ and for all Filippov solutions of the closed-loop system on $[t_0, T]$ with $x(t_0) = x_0$.
In this case, the closed-loop system is called strongly safe for the set $\mathcal{S}$.
}

\end{definition}

\revision{The term strong safety is used since the closed-loop system is safe for all Filippov solutions.}
If $h$ is a locally Lipschitz continuous function, it is also absolutely continuous~\cite{cortes2008discontinuous}.
By \cref{def:filippov-solution}, \revision{any} Filippov solution $x(t)$ is absolutely continuous, and so $(h\circ x)$ is also absolutely continuous and is differentiable almost everywhere~\cite{cortes2008discontinuous}.
Then, the following lemma \revision{(which can be proved using Gronwall Lemma~\cite[Lem.~C.3.1]{sontag2013mathematical})} can be used to guarantee safety:
\begin{lemma}[CBF constraint]~\cite[Lem.~2]{glotfelter2017nonsmooth}
\label{lem:ncbf-safety}
Let $\alpha > 0$, and $h: [t_0,T] \rightarrow \mathbb{R}$ be an absolutely continuous function.
If $h(t_0) > 0$ and
\begin{equation}
\label{eq:ncbf-constraint}
\dot{h}(t) \geq -\alpha\cdot h(t)
\end{equation}
for almost all $t \in [t_0,T]$, then $h(t) \geq h(t_0)e^{-\alpha t} > 0$ $\forall t \in [t_0,T]$.

\end{lemma}

The \revision{strong safety} for the \revision{closed-loop system \cref{eq:n-affine-systems-feedback}} can be enforced by choosing the minimum distance function $h$ as an NCBF and using the CBF constraint \cref{eq:ncbf-constraint}, but, as opposed to previous work on CBFs, the minimum distance $h$ is implicitly computed using an optimization problem.
In the following sections, we will show how to explicitly enforce \cref{eq:ncbf-constraint} for the minimum distance function $h$ for \revision{the collision pair $(\mathcal{C}^i, \mathcal{C}^j)$}.
\revision{
In \cref{sec:minimum-distance-between-strongly-convex-maps-smoothness-properties}, we show that the minimum distance function is locally Lipschitz continuous (\cref{lem:strict-convex-lipschitz-continuity}) and compute the derivatives of the KKT solution (\cref{lem:strict-convex-kkt-solution-c1,prop:strict-convex-kkt-solution-directional-derivative}).
Then, in \cref{sec:ncbfs-for-strictly-convex-sets}, we provide a method to quickly propagate the KKT solutions of the minimum distance problem (\cref{thm:distance-ode}) and propose a CBF-QP-based feedback control law (\cref{thm:cbf-qp}).
We show that the CBF-QP guarantees strong safety for the closed-loop system.
}
A flowchart of the assumptions and results in this paper is shown in \cref{fig:results-flowchart}, and the commonly used symbols are tabulated in \cref{tab:symbols-notations}.

\section{Minimum distance between strongly convex maps and smoothness properties}
\label{sec:minimum-distance-between-strongly-convex-maps-smoothness-properties}

\begin{table}[tbp]
\setlength\extrarowheight{1pt}
\footnotesize
\caption{List of symbols and notations}\label{tab:symbols-notations}
\begin{center}
\begin{tabular}{|l|l||l|l|} \hline
Symbol & Meaning & Symbol & Meaning \\ \hline
$x^i \in \mathcal{X}^i$ & System state of Robot $i$ & $\mathcal{C}^i(x^i)$ & \revision{Safe region of} Robot $i$ at $x^i$ \\
$u^i \in \mathcal{U}^i$ & Input for Robot $i$ & $A^i(x^i, \cdot)$ & Constraints defining $\mathcal{C}^i(x^i)$; \cref{eq:lipschitz-convex-set} \\
$f^i, g^i$ & Dynamics of Robot $i$; \cref{eq:n-affine-systems} & \multirow{2}{*}{$h(x)$} & Square of the minimum distance \\
$F[\cdot]$ & Filippov operator; \cref{eq:filippov-operator-def-limit} & & between $\mathcal{C}^i(x^i)$ and $\mathcal{C}^j(x^j)$; \cref{eq:strict-convex-min-dist} \\ 
$\mathcal{J}^i(x^i, z^i)$ & Active constraints at $z^i$; \cref{eq:strict-convex-index-set-j} & $\mathcal{S}$ & Set of safe states; \cref{eq:safe-set-def} \\
$\lambda^i$ & Dual variables for Robot $i$ & \multirow{3}{*}{$\mathcal{J}^i_k(x^i)$} & Active/strictly active/ \\
$L(x,z,\lambda)$ & Lagrangian function; \cref{eq:strict-convex-lagrangian} & & \revision{degenerate active} set \\
 & & & at $(z^*(x), \lambda^*(x))$; \cref{eq:strict-convex-kkt-active-set} \\
\hline
\end{tabular}
\end{center}
\end{table}

\revision{
The minimum distance function $h$ defines the set of safe states $\mathcal{S}$ and will be used as a nonsmooth control barrier function (NCBF) for collision avoidance (see \cref{subsec:nonsmooth-control-barrier-functions}).
The minimum distance between two sets is computed using an optimization problem.
To use $h$ as an NCBF, we must identify its smoothness properties and compute its derivatives.
In this section, we discuss the applicability of $h$ as an NCBF and compute the derivatives of the KKT solution.
}%

\subsection{Minimum distance problem and its KKT conditions}
\label{subsec:strict-convex-kkt-conditions}

The minimum distance function $h$ for the collision pair $(\mathcal{C}^i, \mathcal{C}^j)$ is given by (restated from \cref{eq:convex-set-min-dist})
\begin{equation}
\label{eq:strict-convex-min-dist}
h(x) = \min_{z} \ \ \{\lVert z^i - z^j \rVert_2^2: \ A^i(x^i, z^i) \leq 0, \ A^j(x^j, z^j) \leq 0\}.
\end{equation}
\revision{
The following result shows that, under \cref{assum:strongly-convex-pair}, the minimum distance problem \cref{eq:strict-convex-min-dist} has a unique optimal solution when $h(x) > 0$.
Further, the minimum distance function $h$ is locally Lipschitz continuous.

\begin{lemma}[Uniqueness of optimal solution and Lipschitz continuity of $h$]
\label{lem:strict-convex-lipschitz-continuity}
Let \cref{assum:strongly-convex-pair} hold and $h(x) > 0$ for some $x \in \mathcal{X}$.
Then, \cref{eq:strict-convex-min-dist} has a unique continuous optimal solution $z^*(x') = (z^{i*}(x'), z^{j*}(x'))$ for all $x'$ in a neighborhood $\mathcal{N}(x)$ of $x$.
Further, the minimum distance function $h$ is locally Lipschitz continuous.

\end{lemma}

\begin{proof}
The proof is provided in \cref{app:proof-strict-convex-lipschitz-continuity}.
\end{proof}
}

\begin{figure}
    \centering
    \label{fig:results-flowchart}
    \includegraphics[width=0.99\linewidth]{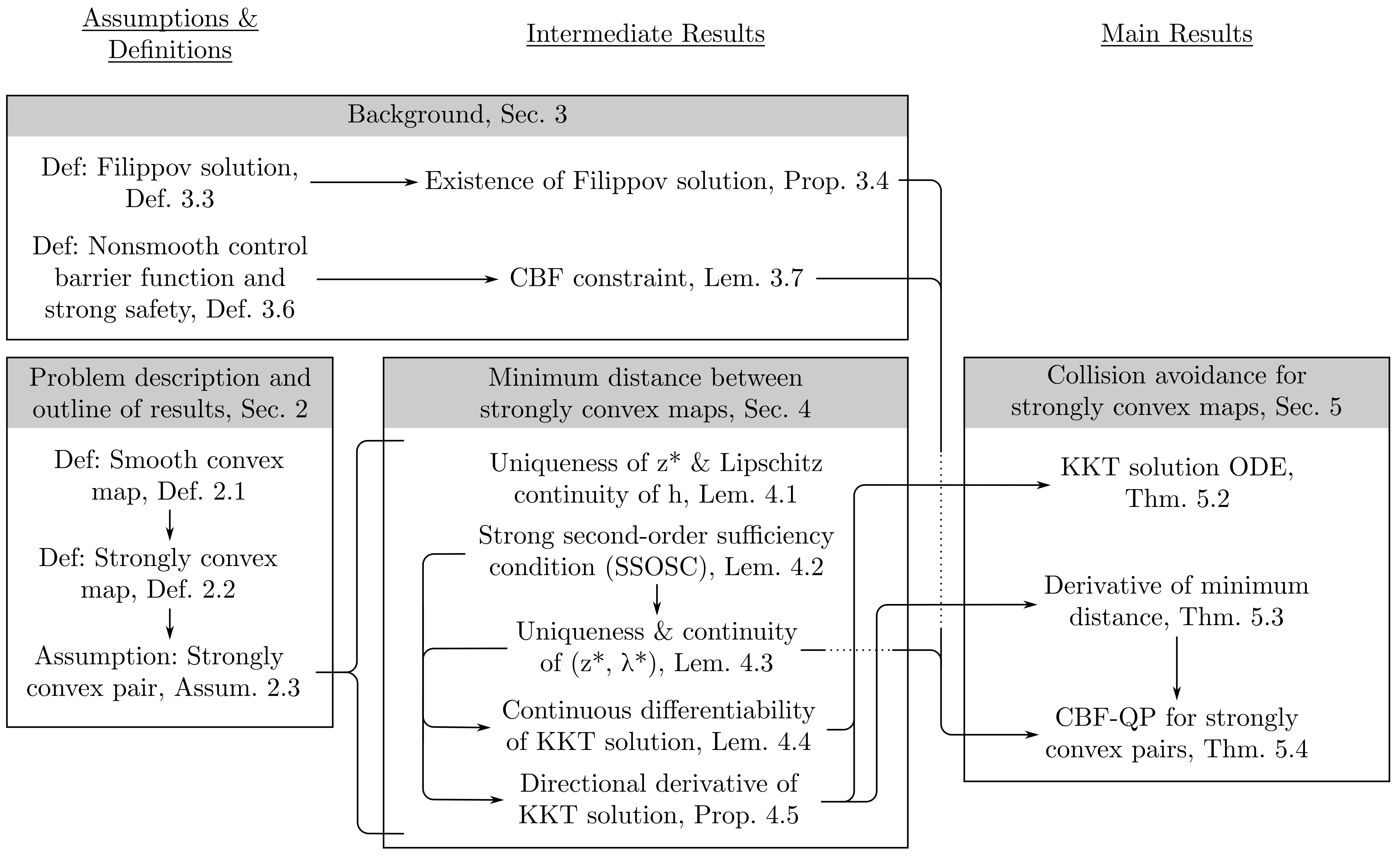}
    \caption{
    A flowchart of the important results in the paper;
    the results in \cref{sec:ncbfs-for-strictly-convex-sets} are the main contributions.
    \Cref{sec:problem-description-outline-main-results} defines smooth convex and strongly convex maps and states the problem considered in the paper.
    \Cref{sec:background} introduces Filippov solutions, nonsmooth control barrier functions (NCBFs), and strong safety for the closed-loop system.
    \Cref{sec:minimum-distance-between-strongly-convex-maps-smoothness-properties} identifies the smoothness properties of the KKT solution of the minimum distance problem.
    Finally, \cref{sec:ncbfs-for-strictly-convex-sets} states the CBF-QP and proves the strong safety property.
    }
\end{figure}

By \cref{lem:strict-convex-lipschitz-continuity}, $h$ is locally Lipschitz continuous and thus can be used as a candidate NCBF (see \cref{def:ncbf}).
To enforce the CBF constraint \cref{eq:ncbf-constraint}, we need to compute the derivative of $h$.
For this, we study the differentiability properties of the KKT solution of \cref{eq:strict-convex-min-dist}.

We can obtain the first-order necessary conditions that $z^*(x)$ must satisfy using KKT conditions.
For a given $x \in \mathcal{X}$, we define the Lagrangian function $L: \mathcal{X} \times \mathbb{R}^{2l} \times \mathbb{R}^{r^i+r^j} \rightarrow \mathbb{R}$ as,~\cite[Chap.~5]{boyd2004convex}
\begin{equation}
\label{eq:strict-convex-lagrangian}
L(x^i, x^j, z^i, z^j, \lambda^i, \lambda^j) := \lVert z^i - z^j \rVert_2^2 + (\lambda^i)^\top A^i(x^i, z^i) + (\lambda^j)^\top A^j(x^j, z^j),
\end{equation}
where $\lambda^i \in \mathbb{R}^{r^i}$ and $\lambda^j \in \mathbb{R}^{r^j}$ are the dual variables corresponding to the inequality constraints of \cref{eq:strict-convex-min-dist}.
We denote the dual variables as $\lambda := (\lambda^i, \lambda^j)$ and the Lagrangian function as $L(x, z, \lambda)$.
We also define $A(x,z) := [A^{i}(x^i,z^i)^\top, A^{j}(x^j,z^j)^\top]^\top$.

The Karush-Kuhn-Tucker (KKT) conditions are necessary optimality conditions for \cref{eq:strict-convex-min-dist}~\cite[Chap.~5]{boyd2004convex}.
The KKT conditions state that, for each $x \in \mathcal{X}$, there exists a KKT solution, $(z^*, \lambda^*)$, such that the following constraints are satisfied:
\begin{subequations}
\label{eq:strict-convex-kkt}
\begin{align}
\nabla_z L(x, z^{*}, \lambda^{*}) & = 0, \label{subeq:strict-convex-kkt-gradient} \\
(\lambda^{*})^\top A(x, z^*) & = 0, \label{subeq:strict-convex-kkt-slackness} \\
\lambda^{*} & \geq 0, \label{subeq:strict-convex-kkt-non-neg} \\
A(x, z^{*}) & \leq 0. \label{subeq:strict-convex-kkt-primal-feasibility}
\end{align}
\end{subequations}
Note that because of the non-negativity \cref{subeq:strict-convex-kkt-non-neg} and primal feasibility \cref{subeq:strict-convex-kkt-primal-feasibility} conditions, $\lambda^{i*}_k A^i_k(x^i, z^{i*}) = 0 \; \forall k \in [r^i]$ (and similarly for $j$).
Thus if the constraint $A^i_k(x^i, z^{i*})$ is inactive at $z^{i*}$, the corresponding dual variable $\lambda^{i*}_k = 0$.
However, both $A^i_k(x^i, z^{i*})$ and $\lambda^{i*}_k$ can be zero simultaneously.
\emph{Strict complementary slackness} condition holds at state $x$ for the KKT solution $(z^*, \lambda^*)$ if for all $k \in [r^i], \lambda^{i*}_k > 0$ whenever $A^i_k(x^i, z^{i*}) = 0$ (and similarly for $j$).

Since the optimization problem \cref{eq:strict-convex-min-dist} is convex and the interior of the feasible set is non-empty (by \cref{def:smooth-convex-set}), the KKT conditions \cref{eq:strict-convex-kkt} are necessary and sufficient conditions for global optimality~\cite[Chap.~5]{boyd2004convex} for each $x \in \mathcal{X}$.
Note that \revision{whenever $h(x) > 0$} there is a unique primal optimal solution for \cref{eq:strict-convex-min-dist} \revision{(by \cref{lem:strict-convex-lipschitz-continuity})}, and so all KKT solutions of \cref{eq:strict-convex-kkt} at $x$ share the same primal optimal solution $z^*(x)$.

\subsection{Smoothness properties of KKT solutions}
\label{subsec:strict-convex-smoothness-kkt-solutions}

To determine the uniqueness, continuity, and differentiability properties of the KKT solution, we use the \emph{strong second-order sufficiency condition} (SSOSC) for \cref{eq:strict-convex-min-dist}.
The SSOSC for \cref{eq:strict-convex-min-dist}~\cite[Eq.~(9)]{jittorntrum1984solution} states that for a KKT solution $(z^*, \lambda^*)$ at state $x$, $\exists a > 0$ such that
\begin{align}
\label{eq:strict-convex-ssosc}
z^\top \nabla_z^2 & L(x, z^*, \lambda^*) z \geq a \lVert z\rVert_2^2, \\\nonumber
\forall z \in \bigl\{z: \; & (z^{i})^\top \nabla_{z^i} A^i_k(x^i, z^{i*}) = 0, \ \forall k \in \mathcal{J}^i(x^i, z^{i*}), \\\nonumber
& (z^{j})^\top \nabla_{z^j} A^j_k(x^j, z^{j*}) = 0, \forall k \in \mathcal{J}^j(x^j, z^{j*}) \bigr\},
\end{align}
where the set of active indices $\mathcal{J}^i$ is defined in \cref{eq:strict-convex-index-set-j}.
SSOSC guarantees that the primal solution $z^*$ minimizes the optimization problem \cref{eq:strict-convex-min-dist} at $x$, but it can also be used to show smoothness properties.
First, we show that SSOSC holds for 
\cref{eq:strict-convex-min-dist}.

\begin{lemma}[Strong second-order sufficiency condition]
\label{lem:strict-convex-ssosc}
Let \revision{\cref{assum:strongly-convex-pair} hold and $x \in \mathcal{X}$ be such that $h(x) > 0$}.
Then, any KKT solution at $x$ satisfies the strong second-order sufficiency condition \cref{eq:strict-convex-ssosc}.

\end{lemma}

\begin{proof}
The proof is provided in \cref{app:proof-strict-convex-ssosc}.
\end{proof}

Now, using the LICQ condition (\cref{subdef:smooth-convex-set-licq}) and the strong second-order sufficiency condition, we can show that there is a unique dual solution for \cref{eq:strict-convex-kkt}.
Moreover, the dual optimal solution is also continuous, as shown next.

\begin{lemma}[Uniqueness and continuity of KKT solution]
\label{lem:strict-convex-unique-continuous-kkt-solution}
Let \revision{\cref{assum:strongly-convex-pair} hold and $x \in \mathcal{X}$ be such that $h(x) > 0$}.
Then there is a unique KKT solution \revision{$(z^*(x'), \lambda^*(x'))$ for \cref{eq:strict-convex-kkt}, for all $x'$ in a neighborhood $\mathcal{N}(x)$ of $x$}.
Moreover, the dual optimal solution \revision{$\lambda^*(\cdot)$ is continuous on $\mathcal{N}(x)$}.

\end{lemma}

\begin{proof}
\revision{
Since the assumptions of \cref{lem:strict-convex-lipschitz-continuity} are satisfied, there is a unique primal optimal solution $z^*(x')$ for all $x'$ in a neighborhood $\mathcal{N}(x)$ of $x$.
Also, let $\mathcal{N}(x)$ be such that $h(x') > 0 \ \forall x' \in \mathcal{N}(x)$.
}
\revision{For any $x' \in \mathcal{N}(x)$}, linear independence condition holds at the primal optimal solution $z^*(x')$ by \cref{subdef:smooth-convex-set-licq}.
\revision{Since $h(x') > 0$}, \cref{lem:strict-convex-ssosc} shows that SSOSC holds for all \revision{$x' \in \mathcal{N}(x)$}.
Finally, by \cref{assum:strongly-convex-pair}, $A^i$ and $A^j$ are twice continuously differentiable (see \cref{subdef:smooth-convex-set-c2}).
Then, \cite[Thm.~2]{jittorntrum1984solution} shows that there is a unique dual optimal solution \revision{$\lambda^*(x')$ for all $x' \in \mathcal{N}(x)$} and that the dual optimal solution $\lambda^*(\cdot)$ is continuous \revision{on $\mathcal{N}(x)$}.
\end{proof}

Note that, for a given state $x$, the KKT conditions \cref{subeq:strict-convex-kkt-gradient,subeq:strict-convex-kkt-slackness} are $2l + r^i + r^j$ equality constraints for the KKT solution $(z^*(x), \lambda^*(x)) \in \mathbb{R}^{2l + r^i + r^j}$.
Suppose the Jacobian of the KKT equality constraints with respect to the KKT solution is invertible.
In that case, we can use the implicit function theorem to guarantee continuous differentiability of the KKT solution $(z^*(\cdot), \lambda^*(\cdot))$.
However, we also need to ensure that the inequalities \cref{subeq:strict-convex-kkt-non-neg,subeq:strict-convex-kkt-primal-feasibility} are satisfied in the vicinity of $x$.
Next, we show that the KKT solution is continuously differentiable whenever strict complementary slackness holds.

\begin{lemma}[Continuous differentiability of KKT solution]
\label{lem:strict-convex-kkt-solution-c1}
Let \revision{\cref{assum:strongly-convex-pair} hold, $x \in \mathcal{X}$ be such that $h(x) > 0$}, and the KKT solution $(z^*(x), \lambda^*(x))$ satisfy the strict complementary slackness condition, \revision{i.e., $\lambda^{i*}_k(x)$ and $A^i_k(x^i, z^{i*}(x))$ are not simultaneously zero for any $k \in [r^i]$ (and likewise for $j$)}.
Then for $x'$ in a neighbourhood of $x$, the unique KKT solution $(z^*(x'), \lambda^*(x'))$ is continuously differentiable.

\end{lemma}

\begin{proof}
\revision{
Since the assumptions of \cref{lem:strict-convex-unique-continuous-kkt-solution} are satisfied, there is a unique KKT solution $(z^*(x'), \lambda^*(x'))$ for all $x'$ in a neighborhood $\mathcal{N}(x)$ of $x$.
Also, let $\mathcal{N}(x)$ be such that $h(x') > 0 \ \forall x' \in \mathcal{N}(x)$.
}%
To show the continuous differentiability property for the KKT solution, we make use of~\cite[Thm.~1]{jittorntrum1984solution}.
The assumptions in~\cite[Thm.~1]{jittorntrum1984solution} are satisfied at $x$ by \revision{the LICQ condition} (\cref{subdef:smooth-convex-set-licq}), \revision{the SSOSC property} (\cref{lem:strict-convex-ssosc}), and \revision{the twice continuous differentiability of $A^i$ and $A^j$} (\cref{subdef:smooth-convex-set-c2}). 
Note that the second-order sufficiency condition (SOSC) is weaker than SSOSC.
Thus, using result (b) from~\cite[Thm.~1]{jittorntrum1984solution}, for all $x'$ in a neighbourhood of $x$ \revision{(different from $\mathcal{N}(x)$)} the KKT solution $(z^*(x'), \lambda^*(x'))$ is continuously differentiable.
\end{proof}

The derivatives of the KKT solution $(z^*(\cdot), \lambda^*(\cdot))$ at $x$ can be computed by differentiating the KKT conditions \cref{eq:strict-convex-kkt} as in ~\cite{jittorntrum1984solution} to obtain,
\begin{subequations}
\label{eq:strict-convex-kkt-solution-derivative}
\begin{gather}
    Q(x) D_x (z^*, \lambda^*)(x) = V(x), \\
    Q(x) := \begin{bmatrix}\nabla_z^2 L & D_z A^\top \\ \text{diag}(\lambda^{*}) D_z A & \text{diag}(A)\end{bmatrix}, \quad V(x) := \begin{bmatrix}-D_x \nabla_z L \\ -\text{diag}(\lambda^{*})D_x A \end{bmatrix}, \label{eq:strict-convex-kkt-solution-derivative-lhs-rhs}
\end{gather}
\end{subequations}
where $Q(x)$ and $V(x)$ are evaluated at $(x, z^*(x), \lambda^*(x))$.
$Q(x)$ is invertible when strict complementary slackness holds \revision{and $h(x) > 0$}.
Thus, we can compute the derivative of $h(x) = \lVert z^{i*}(x)-z^{j*}(x)\rVert^2_2$ as
\begin{equation}
\label{eq:strict-convex-min-dist-derivative-kkt}
\revision{D_x h(x) = 2(z^{i*}(x)-z^{j*}(x))^\top (D_x z^{i*}(x) - D_x z^{j*}(x))}.
\end{equation}

However, when strict complementary slackness does not hold at $x$, $Q(x)$ is not invertible.
For these border cases, we can still obtain the directional derivative of $h$ as the solution of a \revision{linear complementarity problem (LCP)}.
For the optimization problem \cref{eq:strict-convex-min-dist}, we first define the \emph{active set} of constraints $\mathcal{J}_0^i$, the \emph{strictly active set} of constraints $\mathcal{J}_1^i$, \revision{and the \emph{degenerate active set} of constraints $\mathcal{J}_2^i$} for $\mathcal{C}^i$ (and likewise for $\mathcal{C}^j$) at a KKT solution $(z^*(x), \lambda^*(x))$ as
\begin{subequations}
\label{eq:strict-convex-kkt-active-set}%
\begin{align}
\mathcal{J}^i_0(x) & := \{k\in [r^i]: A^i_k(x^i, z^{i*}(x)) = 0\}, & \text{\revision{(active set)}} \\
\mathcal{J}^i_1(x) & := \{k \in [r^i]: \lambda^{i*}_k(x) > 0\}, & \text{\revision{(strictly active set)}} \\
\mathcal{J}^i_2(x) & := \mathcal{J}^i_0(x) \setminus \mathcal{J}^i_1(x). & \text{\revision{(degenerate active set)}} 
\end{align}
\end{subequations}
We also adopt the following notation: If $A^i$ has $r^i$ rows and $A^j$ has $r^j$ rows, then the index for $A = [(A^{i})^\top, (A^{j})^\top]^\top$ is obtained from the set $[r] := [r^i] \sqcup [r^j]$, where $\sqcup$ denotes the disjoint union.
Then, for $(i,k) \in [r]$, $A_{(i,k)} := A^i_k$ and for $(j,k) \in [r]$, $A_{(j,k)} := A^j_k$.
Similarly, we define the index set $\mathcal{J}_0(x) := \mathcal{J}^i_0(x) \sqcup \mathcal{J}^j_0(x)$ (and likewise for $\mathcal{J}_1(x)$ and $\mathcal{J}_2(x)$).
\Cref{fig:strict-convex-illustration} shows an example with active, strictly active, and degenerate active constraints and the index set computation.

\begin{figure}
    \centering
    \label{fig:strict-convex-illustration}
    \includegraphics[width=0.60\linewidth]{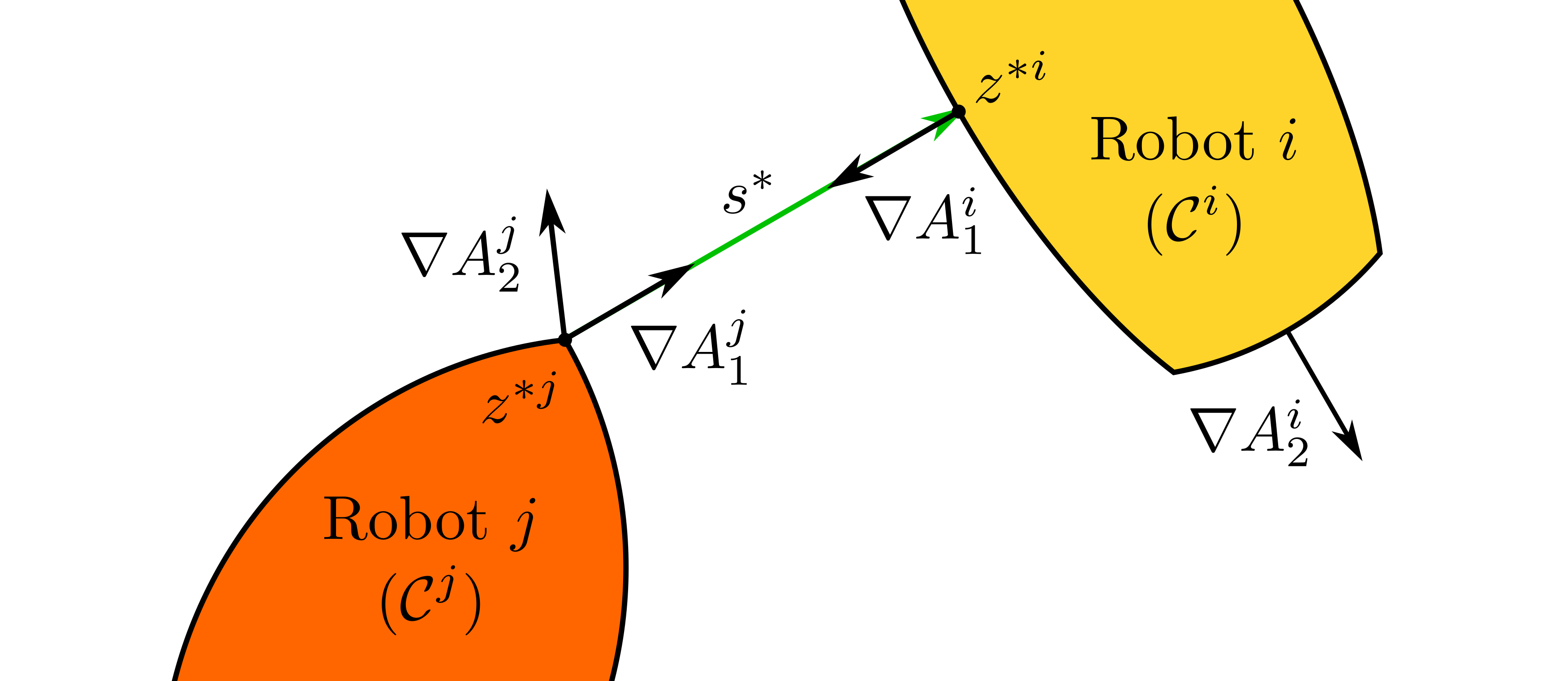}
    \caption{
    Minimum distance problem with its corresponding separating vector and sets of constraints.
    For a given state $x$, the figure illustrates the separating vector $s^*=z^{i*}-z^{j*}$ in green.
    \revision{Under \cref{assum:strongly-convex-pair} and when $h(x) > 0$}, there is a unique optimal solution $z^* = (z^{i*}, z^{j*})$, and the gradients of the constraints, $\nabla_{z^i} A^i_1$ at $z^{i*}$ and $\nabla_{z^j} A^j_1$ and $\nabla_{z^j} A^j_2$ at $z^{j*}$ are shown.
    The KKT condition \cref{subeq:strict-convex-kkt-gradient} indicates that a conic combination of $\nabla_{z^j} A^j_k$ must be equal to $s^*=z^{i*}-z^{j*}$ (the dual variables are the coefficients).
    From the figure, we can see that $s^*$ lies in the cone generated by $\nabla_{z^j} A^j_1$ and $\nabla_{z^j} A^j_2$ (and similarly for $i$), and that $\lambda^{*j}_2 = \lambda^{*i}_2 = 0$.
    Thus the index sets (active set $\mathcal{J}_0$ and the strictly active set $\mathcal{J}_1$) at the state $x$ are $\mathcal{J}^i_0 = \mathcal{J}^i_1 = \{1\}, \mathcal{J}^j_0 = \{1,2\}$, and $\mathcal{J}^j_1 = \{1\}$.
    The combined index sets can be written as: $\mathcal{J}_0 = \{(i,1), (j,1), (j,2)\}$, $\mathcal{J}_1 = \{(i,1), (j,1)\}$, and $\mathcal{J}_2 = \{(j,2)\}$.
    }
\end{figure}

By the complementary slackness condition \cref{subeq:strict-convex-kkt-slackness}, $\mathcal{J}^i_1(x) \subset \mathcal{J}^i_0(x)$, with the equality holding only when strict complementary slackness condition holds.
When strict complementary slackness does not hold, we can use the following result to compute the directional derivatives of the KKT solution, obtained as an extension of \cref{lem:strict-convex-kkt-solution-c1}.
\begin{proposition}[Directional derivative of KKT solution]~\cite[Thm.~4]{jittorntrum1984solution}
\label{prop:strict-convex-kkt-solution-directional-derivative}
Let \revision{\cref{assum:strongly-convex-pair} hold, $x \in \mathcal{X}$ be such that $h(x) > 0$}, and $\mathring{x} \in \mathbb{R}^{2n}$ be a direction of perturbation from $x$.
Consider the following set of constraints for $(\mathring{z}, \mathring{\lambda}) \revision{\in \mathbb{R}^{2l+r^i+r^j}}$,
\begin{subequations}
\label{eq:strict-convex-kkt-right-derivative}%
\begin{alignat}{2}
& \nabla_z^2L \mathring{z} + D_z A^\top \mathring{\lambda} = -D_x (\nabla_z L)[\mathring{x}], \span\span \label{subeq:strict-convex-kkt-right-derivative-gradient} \\
& \nabla_z A_k^\top \mathring{z} = -\nabla_x A_k^\top \mathring{x}, && k \in \mathcal{J}_1(x), \label{subeq:strict-convex-kkt-right-derivative-feasibility} \\
& \nabla_z A_k^\top \mathring{z} \leq -\nabla_x A_k^\top \mathring{x}, && k \in \mathcal{J}_2(x), \\
& \mathring{\lambda}_k = 0, \ \ k \in \mathcal{J}_0(x)^c, && \mathring{\lambda}_k \geq 0, \ \ k \in \mathcal{J}_2(x), \label{subeq:strict-convex-kkt-right-derivative-non-neg}\\
& \mathring{\lambda}_k(\nabla_z A_k^\top \mathring{z} + \nabla_x A_k^\top \mathring{x}) = 0, \quad\quad && k \in \mathcal{J}_2(x), \label{subeq:strict-convex-kkt-right-derivative-slackness}%
\end{alignat}%
\end{subequations}%
where $(\cdot)^c$ represents the complement of a set.
The system of equalities and inequalities \cref{eq:strict-convex-kkt-right-derivative} is evaluated at $(x, z^*(x), \lambda^*(x))$.

Then, \cref{eq:strict-convex-kkt-right-derivative} has a unique solution, $(\mathring{z}^*, \mathring{\lambda}^*)$, which is the directional derivative of $(z^*(\cdot), \lambda^*(\cdot))$ at $x$ along $\mathring{x}$.

\end{proposition}

\begin{proof}
\revision{
Since the assumptions of \cref{lem:strict-convex-unique-continuous-kkt-solution,lem:strict-convex-ssosc} are satisfied, there is a unique KKT solution $(z^*(x), \lambda^*(x))$ at $x$ and SSOSC is satisfied at $x$.
}
The proof follows from~\cite[Thm.~4]{jittorntrum1984solution}, which has the requirement that the strong second-order sufficient condition (SSOSC) \cref{eq:strict-convex-ssosc} be satisfied for the unique KKT solution.
\end{proof}

\Cref{prop:strict-convex-kkt-solution-directional-derivative} provides a method to calculate the directional derivative of the KKT solution along a direction $\mathring{x}$ by solving the LCP \cref{eq:strict-convex-kkt-right-derivative}.
\revision{
In \cref{sec:ncbfs-for-strictly-convex-sets}, we will propose modifications of \cref{eq:strict-convex-kkt-solution-derivative} and \cref{eq:strict-convex-kkt-right-derivative}, which will allow us to compute KKT solutions quickly.
}%
\revision{We additionally note that when strict complementary slackness holds, i.e., when the degenerate active set $\mathcal{J}_2(x) = \emptyset$ (see \cref{eq:strict-convex-kkt-active-set})}, \cref{eq:strict-convex-kkt-right-derivative} reduces to \cref{eq:strict-convex-kkt-solution-derivative}.

\revision{To summarize the results in this section, \cref{lem:strict-convex-lipschitz-continuity} shows that the minimum distance function $h$ is locally Lipschitz continuous and \cref{lem:strict-convex-unique-continuous-kkt-solution} shows that the KKT solution is unique and continuous.}
\Cref{lem:strict-convex-kkt-solution-c1,prop:strict-convex-kkt-solution-directional-derivative} provide methods to compute the directional derivative $(\mathring{z}, \mathring{\lambda})$ of the KKT solution $(z^*(\cdot), \lambda^*(\cdot))$ along $\mathring{x}$, when the \revision{degenerate active set} $\mathcal{J}_2(x)$ is empty and non-empty respectively.
In the next section, we use these results to show that the minimum distance between a strongly convex pair can be quickly computed using an ODE.
Further, we can compute the derivative of the minimum distance and use it in a CBF-QP to guarantee \revision{strong safety (see \cref{def:ncbf})}.
An outline of the results in this section and the following section is provided in the flowchart \cref{fig:results-flowchart}.

\section{Main Result: Collision avoidance for strongly convex maps}
\label{sec:ncbfs-for-strictly-convex-sets}

This section shows three main results: 
First, \cref{thm:distance-ode} provides a \revision{numerical method to quickly propagate the KKT solutions of the minimum distance problem along a state trajectory; 
}%
second, \cref{thm:minimum-distance-derivative} computes the derivative of the minimum distance using the KKT solution; 
and third, \cref{thm:cbf-qp} provides a CBF-QP formulation that guarantees \revision{strong safety (see \cref{def:ncbf})} between the strongly convex pair.

We start with the following result, which summarizes \cref{lem:strict-convex-kkt-solution-c1,prop:strict-convex-kkt-solution-directional-derivative} and computes the directional derivative of the KKT solution $(z^*(x), \lambda^*(x))$.
The directional derivatives can then be used \revision{to integrate the KKT solution numerically, given the state derivative $\dot{x}(t)$}.
Additional stabilizing terms are added to ensure the numerical KKT solution locally stabilizes to the actual KKT solution.

\begin{theorem}[KKT solution ODE]%
\label{thm:distance-ode}%
Let \revision{\cref{assum:strongly-convex-pair} hold and $x(t) \in \mathcal{X}$, $t \geq t_0$}, be a differentiable state trajectory such that $\revision{h(x(t)) > 0} \; \forall t \geq t_0$.
Then, the KKT solutions $(z^*(t), \lambda^*(t)) := (z^*(x(t)), \lambda^*(x(t)))$, $t \geq t_0$, can be obtained 
as a solution of the following ODE: 
At time $t = t_0$, let $(z^*(t_0), \lambda^*(t_0))$ be initialized by solving the optimization problem \cref{eq:strict-convex-min-dist}.
At time $t$, $(\dot{z}^*(t), \dot{\lambda}^*(t))$ can be computed, depending on \revision{the degenerate active set of constraints $\mathcal{J}_2(x(t))$ (see \cref{eq:strict-convex-kkt-active-set})}, as follows:

\begin{enumerate}[labelindent=0pt, itemindent=!, leftmargin=*]
\item When $\mathcal{J}_2(x(t)) = \emptyset$,
\begin{equation}
\label{eq:distance-ode-empty-j2}
    \begin{bmatrix}\dot{z}^*(t) \\ \dot{\lambda}^*(t)\end{bmatrix} = Q(x(t))^{-1} \bigl(V(x(t)) \dot{x}(t) - \kappa e_{kkt}(x(t))\bigr),
\end{equation}
where $Q$ and $V$ are defined in \cref{eq:strict-convex-kkt-solution-derivative-lhs-rhs}, $\kappa > 0$ is a stabilizing constant, and
\begin{equation}
\label{eq:distance-ode-e-kkt}
    e_{kkt}(x) := \begin{bmatrix}
        \nabla_z L(x, z^*(x), \lambda^*(x)) \\
        \text{diag}(\lambda^*(x)) A(x, z^*(x))
    \end{bmatrix}.
\end{equation}
The term $e_{kkt}(x(t))$ is the residual error in the KKT conditions \cref{eq:strict-convex-kkt} at time $t$.

\item \label{thm:distance-ode-subcase-nonempty-j2}
When $\mathcal{J}_2(x(t)) \neq \emptyset$, $\dot{z}^*(t)$ is the unique optimal solution of the following QP.
\begin{subequations}
\label{eq:distance-ode-j2}%
\begin{alignat}{3}
   \min_{\mathring{z}} \quad & (1/2)\mathring{z}^\top Q_{11}(x(t)) \mathring{z} + [V_1(x(t))\dot{x}(t) + \tilde{e}_{1, kkt}(x(t))]^\top \mathring{z} \span\\
    \text{s.t.} \quad & P_{eq}(x(t)) \mathring{z} = q_{eq}(x(t))\dot{x}(t) + \tilde{e}_{2, kkt}(x(t)), \label{subeq:distance-ode-j2-eq} \\
    & P_{in}(x(t)) \mathring{z} \leq q_{in}(x(t)) \dot{x}(t), \label{subeq:distance-ode-j2-ineq}
\end{alignat}%
\end{subequations}%
where
\begin{equation}
\label{eq:distance-ode-j2-consts}
\begin{gathered}
    Q_{11}(x) := \nabla_z^2 L, \quad V_{1}(x) := D_x \nabla_z L, \quad \tilde{e}_{1, kkt}(x) := \kappa \nabla_z L - \kappa D_z A_{\mathcal{J}_{0c}}^\top \lambda^*_{\mathcal{J}_{0c}}, \\
    P_{eq}(x) := D_z A_{\mathcal{J}_1}, \quad q_{eq}(x) := -D_x A_{\mathcal{J}_1}, \quad \tilde{e}_{2, kkt}(x) := -\kappa A_{\mathcal{J}_1}, \\
    \quad P_{in}(x) := D_z A_{\mathcal{J}_2}, \quad q_{in}(x) := -D_x A_{\mathcal{J}_2},
\end{gathered}
\end{equation}
\revision{where $\mathcal{J}_1$ is the strictly active set of constraints and $\mathcal{J}_{0c}$ (the inactive set of constraints) denotes the complement of the active set $\mathcal{J}_0$ (see \cref{eq:strict-convex-kkt-active-set})}.
All variables \revision{and the sets $\mathcal{J}_1, \mathcal{J}_2$, and $\mathcal{J}_{0c}$} in \cref{eq:distance-ode-j2-consts} are evaluated at $(x, z^*(x), \lambda^*(x))$.
The dual derivative $\dot{\lambda}^*(t)$ is given by,
\begin{equation}
\label{eq:distance-ode-j2-dual}
\dot{\lambda}^*_{\mathcal{J}_1}(t) = \lambda^*_{eq}, \quad \dot{\lambda}^*_{\mathcal{J}_2}(t) = \lambda^*_{in}, \quad \dot{\lambda}^*_{\mathcal{J}_{0c}}(t) = -\kappa \lambda^*_{\mathcal{J}_{0c}}(t),
\end{equation}
where $\lambda^*_{eq}$ and the $\lambda^*_{in}$ are the optimal dual solutions corresponding to \cref{subeq:distance-ode-j2-eq} and \cref{subeq:distance-ode-j2-ineq} respectively.
\revision{Note that $\mathcal{J}_1 \cup \mathcal{J}_{2} \cup \mathcal{J}_{0c} = (\{i\} \times [r^i]) \cup (\{j\} \times [r^j])$ at any state\footnote{\label{footnote:ri-rj}
Recall that $r^i$ is the number of constraints defining $\mathcal{C}^i(x^i)$, \cref{eq:lipschitz-convex-set}.
}; in other words, \cref{eq:distance-ode-j2-dual} fully defines $\dot{\lambda}^*(t)$}.
\end{enumerate}

\end{theorem}

\begin{proof}
The proof is provided in \cref{app:proof-distance-ode}.
\end{proof}

\Cref{thm:distance-ode} shows that the derivative of the KKT solution can be computed as the solution to a system of linear equations \cref{eq:distance-ode-empty-j2} when \revision{the degenerate active set} $\mathcal{J}_2(x(t)) = \emptyset$. 
\revision{When $\mathcal{J}_2(x(t)) \neq \emptyset$, the QP \cref{eq:distance-ode-j2} can be solved as an LCP in the variable $\dot{\lambda}^*_{\mathcal{J}_2}$ (this can be done by writing the KKT conditions of the QP \cref{eq:distance-ode-j2} and eliminating the variables $\mathring{z}, \dot{\lambda}^*_{\mathcal{J}_1}$, and $\dot{\lambda}^*_{\mathcal{J}_{0_c}}$; see the proof of \cref{thm:distance-ode}).}
%
%
Thus, \cref{thm:distance-ode} provides a way to integrate the time-varying KKT solution numerically instead of solving the distance optimization at each time.

\revision{
\begin{remark}[KKT solution propagation]
The state trajectory $x(t)$ of the closed-loop system \cref{eq:n-affine-systems-feedback} is obtained as a Filippov solution (see \cref{def:filippov-solution}).
So, $x(t)$ is an absolutely continuous function of $t$ and is differentiable almost everywhere.
The KKT solution ODE in \cref{thm:distance-ode} assumes that the state trajectory $x(t)$ is differentiable, which is not guaranteed for the closed-loop system.
Nevertheless, \cref{thm:distance-ode} is useful in discrete-time control applications to compute the KKT solution at the next time step given the current KKT solution and the control input (see \cref{sec:results}).
\end{remark}
}

To impose a CBF constraint on the minimum distance, we need to compute the derivative of $h$ with respect to the state $x$.
The following result provides a method to compute the derivative of the minimum distance.

\begin{theorem}[Derivative of minimum distance]%
\label{thm:minimum-distance-derivative}%
\revision{Let \cref{assum:strongly-convex-pair} hold and $h(x) > 0$ at $x \in \mathcal{X}$.
Then, $h$ is continuously differentiable in a neighborhood $\mathcal{N}(x)$ of $x$ and}
\begin{equation}
\label{eq:minimum-distance-derivative}
    D_x h(x) = \lambda^*(x)^\top D_x A(x, z^*(x)).
\end{equation}
Also, $D_x h$ is directionally differentiable \revision{at $x$}.

Further, if the functions $A^i$ and $A^j$ defining $\mathcal{C}^i$ and $\mathcal{C}^j$ respectively are smooth and $r^i = r^j = 1$ \footref{footnote:ri-rj}, then the minimum distance $h$ between $\mathcal{C}^i$ and $\mathcal{C}^j$ is \revision{a smooth function on $\mathcal{N}(x)$}.

\end{theorem}

\begin{proof}
The proof is provided in \cref{app:proof-minimum-distance-derivative}.
\end{proof}

In \cref{thm:minimum-distance-derivative}, the requirement that $r^i = 1$, i.e., that the number of constraints defining $\mathcal{C}^i(x^i)$ be one, is not too restrictive.
When multiple constraints define the set $\mathcal{C}^i(x^i)$, the softmax function can be used to tightly overapproximate $\mathcal{C}^i(x^i)$ using a single function.

Finally, we can use \cref{thm:minimum-distance-derivative} to enforce the CBF constraints \cref{eq:ncbf-constraint} and formulate a CBF-QP that guarantees \revision{strong safety (see \cref{def:ncbf})} between a strongly convex pair.

\begin{theorem}[CBF-QP for strongly convex pairs]%
\label{thm:cbf-qp}%
Let \revision{\cref{assum:strongly-convex-pair} hold, $x_0 \in \mathcal{X}$ be such that $h(x_0) > 0$}, $u^{nom}: \mathcal{X} \rightarrow \mathcal{U}$ be any nominal controller that does not guarantee the safety of the system, and $\Phi \succeq 0$.
Consider the following CBF-QP, which is used to \revision{compute a feedback control law $u^*_{fb}$} for the dynamical system.
\begin{subequations}
\label{eq:cbf-qp}%
\begin{align}
   u^*_{fb}(x) \in \ \argmin_{u \in \mathcal{U}} \quad & \lVert u - u^{nom}(x)\rVert_\Phi^2 \\
    \text{s.t.} \quad & \lambda^*(x)^\top D_x A(x, z^*(x)) (f(x) + g(x)u) \geq -\alpha h(x) \label{subeq:cbf-qp-cbf},%
\end{align}%
\end{subequations}%
where $\alpha > 0$ is the CBF constant (see \cref{lem:ncbf-safety}).

Assume that \cref{eq:cbf-qp} is feasible for all $x$ \revision{such that $h(x) > 0$}, for some $\alpha > 0$.
Then, any measurable feedback control law $u^*_{fb}$ obtained as a solution of \cref{eq:cbf-qp} makes the closed-loop system \revision{strongly safe, i.e., $h(x(t)) > 0$ for all $t \in [t_0, T]$ and Filippov solutions of the closed-loop system on $[t_0, T]$ with $T > t_0$ and $x(t_0) = x_0$}.
The \revision{closed-loop system's strong safety property} is independent of the cost function used in \cref{eq:cbf-qp}.

\end{theorem}

\begin{proof}
The proof is provided in \cref{app:proof-cbf-qp}.
\end{proof}

Note that in the CBF-QP \cref{eq:cbf-qp}, \cref{subeq:cbf-qp-cbf} is the CBF constraint 
\cref{eq:ncbf-constraint}.

\begin{remark}[Nonsmooth CBFs and safety via feasibility]
Many existing works in the CBF literature assume Lipschitz continuity of the feedback controller $u^*_{fb}$ to guarantee the uniqueness and continuous differentiability of the closed-loop state trajectory $x$~\cite[Thm.~2]{ames2019control}.
However, Lipschitz continuity of the optimal solution of a parametric QP can be challenging to prove or guarantee in practice~\cite{morris2015continuity}.
Further, such results require the CBF-QP to be solved to optimality.
This is undesirable because, ideally, we want the system's safety to depend on the feasibility of the CBF-QP and not the optimality.

In this paper, we forego the uniqueness property of the closed-loop state trajectory to guarantee safety for any feasible solution to the CBF-QP rather than just for the optimal solution.
This way, we can use any continuous cost function in \cref{eq:cbf-qp} \revision{(even a non-convex one)} and still guarantee safety for any feasible solution.
However, the discontinuous input can result in chatter in the closed-loop trajectory~\cite{morris2015continuity}.

\end{remark}

\revision{
\begin{remark}[Strongly convex map algebra]
\label{rem:strongly-convex-map-algebra}
Our collision avoidance formulation based on smooth convex and strongly convex maps allows for many convex set algebraic manipulations.
To see this, we first consider an extension of the minimum distance problem \cref{eq:convex-set-min-dist} as
\begin{equation}
\label{eq:convex-set-min-dist-extension}
h(x) = \ \min_{z} \ \ \bigl\{\lVert P^i z^i - P^j z^j \rVert_M^2: \ z^i \in \mathcal{C}^i(x^i), \ z^j \in \mathcal{C}^j(x^j)\bigr\},
\end{equation}
where $P^i$, $P^j$ are projection matrices with full row-rank and $M \succ 0$ is the inner product metric.
Under \cref{assum:strongly-convex-pair}, all of the results in this paper apply to \cref{eq:convex-set-min-dist-extension} as well.
The Cartesian product and intersection of two smooth convex maps $\mathcal{C}^{i1}$ and $\mathcal{C}^{i2}$ can be represented using the functions $(A^{i1}(x^i, z^{i1}), A^{i2}(x^i, z^{i2}))$ and $(A^{i1}(x^i, z^{i}), A^{i2}(x^i, z^{i}))$ respectively.
Thus, our formulation implicitly supports projections, Cartesian products, Minkowski sums (using Cartesian products with projections), and intersections.
The convex hull corresponding to two smooth convex maps can be explicitly formulated using perspective functions~\cite[Ex.~4.56]{boyd2004convex}.
For intersections and convex hulls, \cref{subdef:smooth-convex-set-licq,subdef:smooth-convex-set-slater} must be explicitly checked.
\end{remark}
}

\begin{remark}[CBF-QP for multiple collision pairs]
The CBF-QP formulation \cref{eq:cbf-qp} is shown for a single strongly convex pair.
When there are multiple robots and obstacles (and thus multiple strongly convex pairs), the CBF constraint \cref{subeq:cbf-qp-cbf} is used for each pair.
Note that only one CBF constraint is needed per collision pair.
\revision{
Further, the strong safety property of the CBF-QP \cref{eq:cbf-qp} holds even for multiple collision pairs.
This is because the proof of \cref{thm:cbf-qp} uses the upper semi-continuity property of the CBF-QP feasible set, which remains valid when multiple CBF constraints are used~\cite[Ex.~5.8]{rockafellar2009variational}.
}
\end{remark}

\begin{remark}[Compatibility with other CBF methods]
Although the CBF-QP formulation \cref{eq:cbf-qp}, as presented, is restricted to centralized systems with no uncertainty and for which the distance function is first-order with respect to the inputs (i.e., relative degree $1$ systems), these are not restrictions on our method.
Existing methods for decentralized and distributed CBFs, such as \cite{wang2017safety}, are compatible with our method.
\revision{CBFs for uncertain systems and systems with disturbance can also be rephrased as CBFs for convex sets, where a convex set represents the uncertainty bound~\cite{cohen2022robust} (see \cref{subsec:example-kkt-solution-ode-verification})}.
Using the backup CBF method \cite{chen2021backup} or by system-specific heuristics, our method can be applied to higher-order systems (see \cref{subsec:example-cbf-obstacle avoidance}).
\revision{Additionally, when a single constraint defines the safe region maps $\mathcal{C}^i$ and $\mathcal{C}^j$, i.e., when $r^i = r^j = 1$, \cref{thm:minimum-distance-derivative} shows that the minimum distance function $h$ is smooth; thus, higher-order CBF methods~\cite{xiao2022higher} can be directly applied.}
\end{remark}

\section{Results}
\label{sec:results}

In this section, we provide simulation results to demonstrate the validity of the three main theorems in the previous section.
In the first example, we show \revision{that \cref{thm:distance-ode} can be used to significantly speed up the computation of the KKT solutions along a state trajectory}.
The KKT solutions computed using \cref{thm:distance-ode} are compared to KKT solutions directly obtained using a solver.
The distance derivatives from \cref{thm:minimum-distance-derivative} are also compared to the actual distance derivative.
In the second example, we solve an obstacle avoidance problem using the CBF-QP formulation in \cref{thm:cbf-qp}.
\revision{We demonstrate that our method allows for real-time collision avoidance between strongly convex pairs.}

\revision{
The simulations are performed on Ubuntu $20.04$ on a laptop with an Intel Core $i7-10870H$ CPU $@ \SI{2.20}{GHz}$.
The nonlinear optimization solver \emph{Ipopt} \cite{wachter2006implementation} is used to compute the actual KKT solutions of the minimum distance problem \cref{eq:strict-convex-min-dist} with warm start enabled.
The \emph{Eigen3} linear algebra library \cite{eigen2010} is used for matrix computations.
The QP \cref{eq:distance-ode-j2} in the KKT solution ODE (when the degenerate active set $\mathcal{J}_2$ is nonempty) is solved by first reformulating the problem in terms of the variable $\dot{\lambda}^*_{\mathcal{J}_2}$ and then using a custom lexicographic Lemke solver \cite[Ch.~2]{murty1988linear} to solve the resulting LCP problem (also see the proof of \cref{thm:distance-ode}).
The CBF-QP \cref{eq:cbf-qp} is solved using \emph{OQSP} \cite{stellato2020osqp}.
The 3D visualizations for the examples are generated using \emph{Meshcat-python} \cite{threejs,meshcat-python}.
The C\texttt{++} library and simulation code for the following examples can be found in the repository\footref{code}.
}

\subsection{Example 1: KKT solution ODE verification}
\label{subsec:example-kkt-solution-ode-verification}

In this subsection, we verify \cref{thm:distance-ode,thm:minimum-distance-derivative} \revision{and show that they can be used to quickly compute the KKT solution and distance derivative along a state trajectory}.
\revision{
The simulation setup consists of a static polytope $\mathcal{C}^{\text{obs}}$ and a quadrotor system with state $x^q = (p^q, v^q, R^q)$, where $(p^q, R^q) \in SE(3)$ is the orientation of the quadrotor and $v^q \in \mathbb{R}^3$ is the translational velocity of the CoM.
The inputs to the quadrotor system are given by $u^q = (T^q, \omega^q)$, where $T^q \in \mathbb{R}_+$ is the net thrust applied by the quadrotor and $\omega^q \in \mathbb{R}^3$ is body frame angular velocity of the quadrotor\footnote{
Generally, a low-level onboard controller uses the net thrust $T$ and body frame angular velocity $\omega$ to compute the individual rotor thrust.
Thus, we don't consider the dynamics of the full quadrotor system.
}.
For the rest of this section, we suppress the superscript `$q$' for the quadrotor system.
The quadrotor system dynamics are given by \cite{lee2010geometric} as
\begin{equation}
\label{eq:quadrotor-dynamics}
\begin{bmatrix}\dot{p} \\ \dot{v} \\ \text{vec}(\dot{R})\end{bmatrix} = \begin{bmatrix}v \\ -ge_3 \\ 0_{9}\end{bmatrix} + \begin{bmatrix}0_{3} & 0_{3 \times 3} \\ \frac{1}{m} Re_3 & 0_{3 \times 3} \\ 0_{9} & \hat{R} \end{bmatrix} \begin{bmatrix}T \\ \omega\end{bmatrix},
\end{equation}
where $g$ is the gravitational acceleration constant, $m$ is the mass of the quadrotor, $e_3 = (0, 0, 1)$, $\text{vec}(\cdot)$ vectorizes a matrix by stacking its columns, $\hat{R} \in \mathbb{R}^{9 \times 3}$ is such that $\hat{R}w = \text{vec}(R\hat{w})$, and $\hat{\omega}$ is defined such that $\hat{\omega} z = w \times z \ \forall z \in \mathbb{R}^3$.

The safe region for the obstacle is directly chosen as the obstacle set $\mathcal{C}^{\text{obs}}$.
The safe region $\mathcal{C}$ for the quadrotor is computed in two steps, as shown in \cref{subfig:quadrotor-uncertainty-safe-set}.
First, we approximate the shape of the quadrotor at state $x = (p, v, R)$ as follows:
\begin{equation}
\label{eq:example-quadrotor-shape-set}
\mathcal{C}^1(x) := \{z \in \mathbb{R}^3: \bar{z} = R^\top(z - p), \ |\bar{z}_1|^{2.5} + |\bar{z}_2|^{2.5} + |\bar{z}_3 / 0.4|^{2.5} + 10^{-2} \lVert\bar{z}\rVert^2_2 \leq 0.3\}.
\end{equation}
Next, we assume that the position of the quadrotor is subject to measurement uncertainties, which is given by an ellipsoidal uncertainty set $\mathcal{C}^2(x)$ defined as
\begin{equation*}
\mathcal{C}^2(x) := \{z \in \mathbb{R}^3: 3z_1^2 + 2z_1 z_2 + 2 z_2^2 + z_3^2 \leq 1 + (8/5\pi) \tan^{-1}(z_3)\}.
\end{equation*}
Note that the size of the uncertainty set depends on the $z$-coordinate of the quadrotor position.
For example, this can be the case when the quadrotor position estimate depends on ground measurements.
The sets $\mathcal{C}^1(x)$ and $\mathcal{C}^2(x)$ are depicted in \cref{subfig:quadrotor-uncertainty-safe-set}.
Further, we can also verify that $\mathcal{C}^1$ and $\mathcal{C}^2$ are strongly convex maps (see \cref{def:strongly-convex-set}).
Finally, the safe region $\mathcal{C}(x)$ of the quadrotor at state $x$ is given by the Minkowski sum of $\mathcal{C}^1(x)$ and $\mathcal{C}^2(x)$, and thus $\mathcal{C}$ is a strongly convex map (see \cref{rem:strongly-convex-map-algebra}).
Then, $(\mathcal{C}, \mathcal{C}^\text{obs})$ forms a strongly convex pair (see \cref{assum:strongly-convex-pair}).

\begin{figure}[tbp]
    \centering
    \hspace*{\fill}%
    \subfloat[
    Snapshots of the quadrotor system with the quadrotor safe region (pink mesh) and the static obstacle polytope (grey).
    ]{
    \label{subfig:example-kkt-ode-env}
    {
    \includegraphics[width=2.20in]{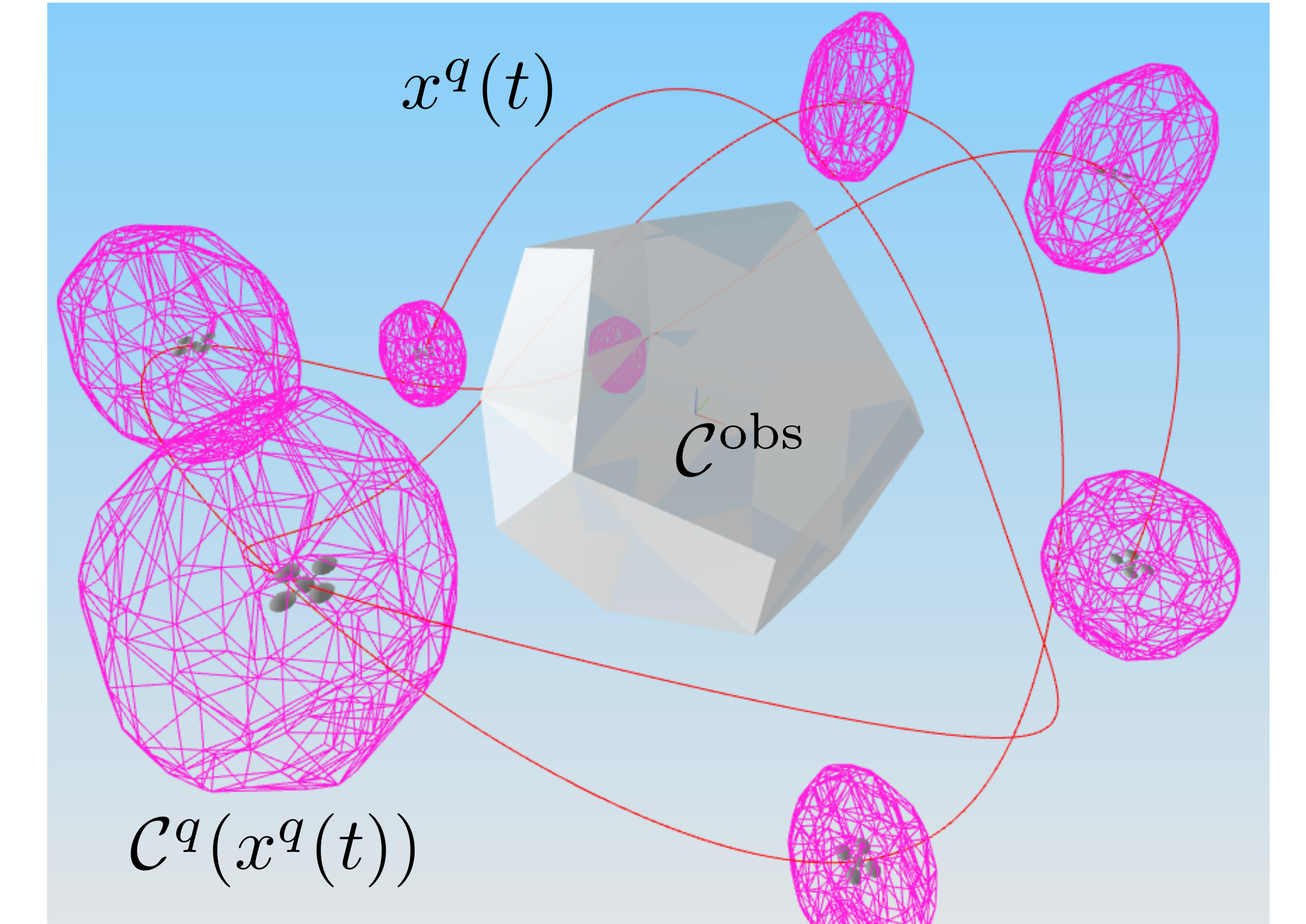} 
    }%
    }%
    \centerhfill%
    \subfloat[
    Relative error of the KKT solution.
    The $y$-axis is in the log scale.
    ]{
    \label{subfig:example-kkt-ode-kkt}
    {
    \includegraphics[width=2.20in]{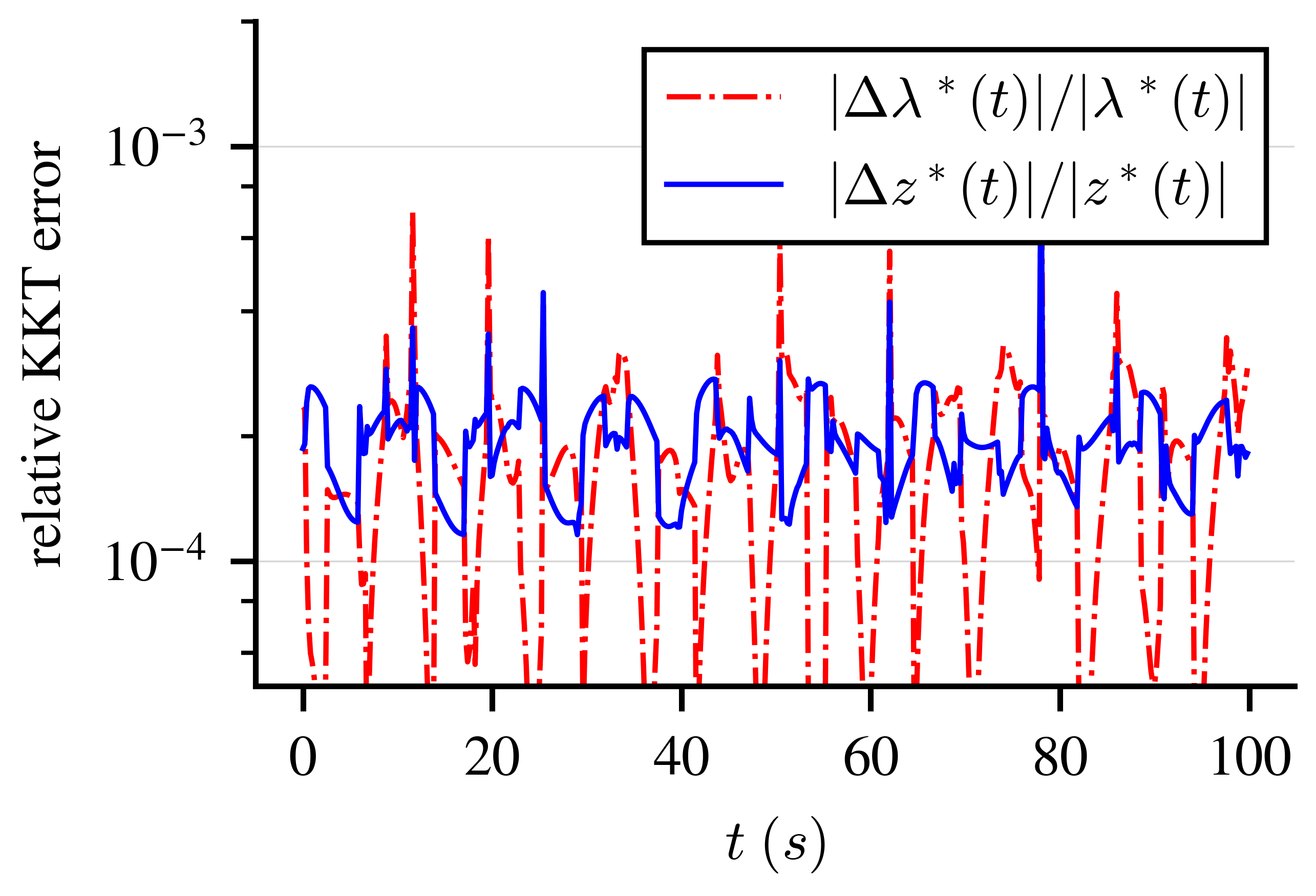}
    }%
    }%
    \hspace*{\fill}%

    \hspace*{\fill}%
    \subfloat[
    Relative error of the minimum distance.
    ]{
    \label{subfig:example-kkt-ode-dist}
    {
    \includegraphics[width=2.20in]{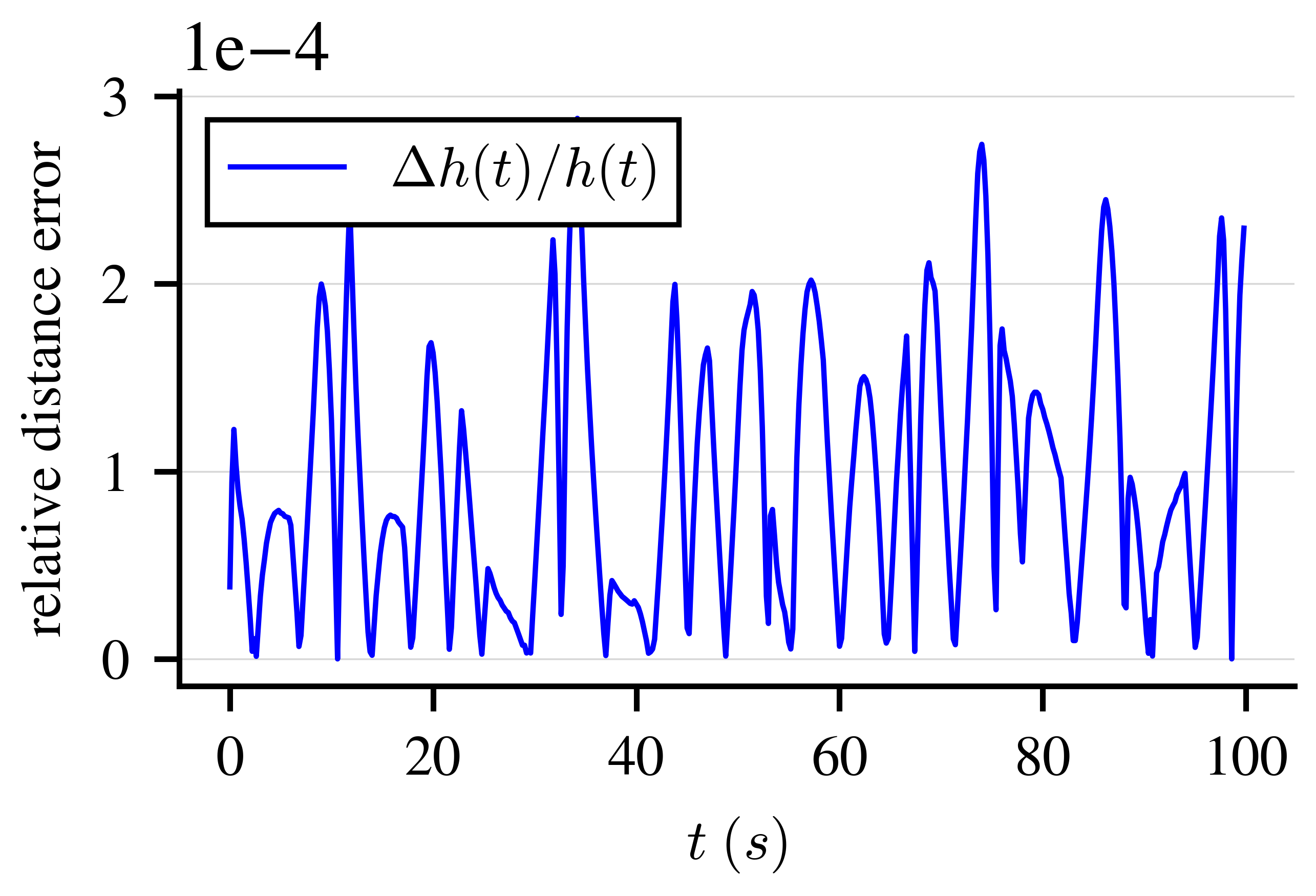}
    }%
    }%
    \centerhfill%
    \subfloat[
    Error in minimum distance derivative.
    ]{
    \label{subfig:example-kkt-ode-Ddist}
    {
    \includegraphics[width=2.20in]{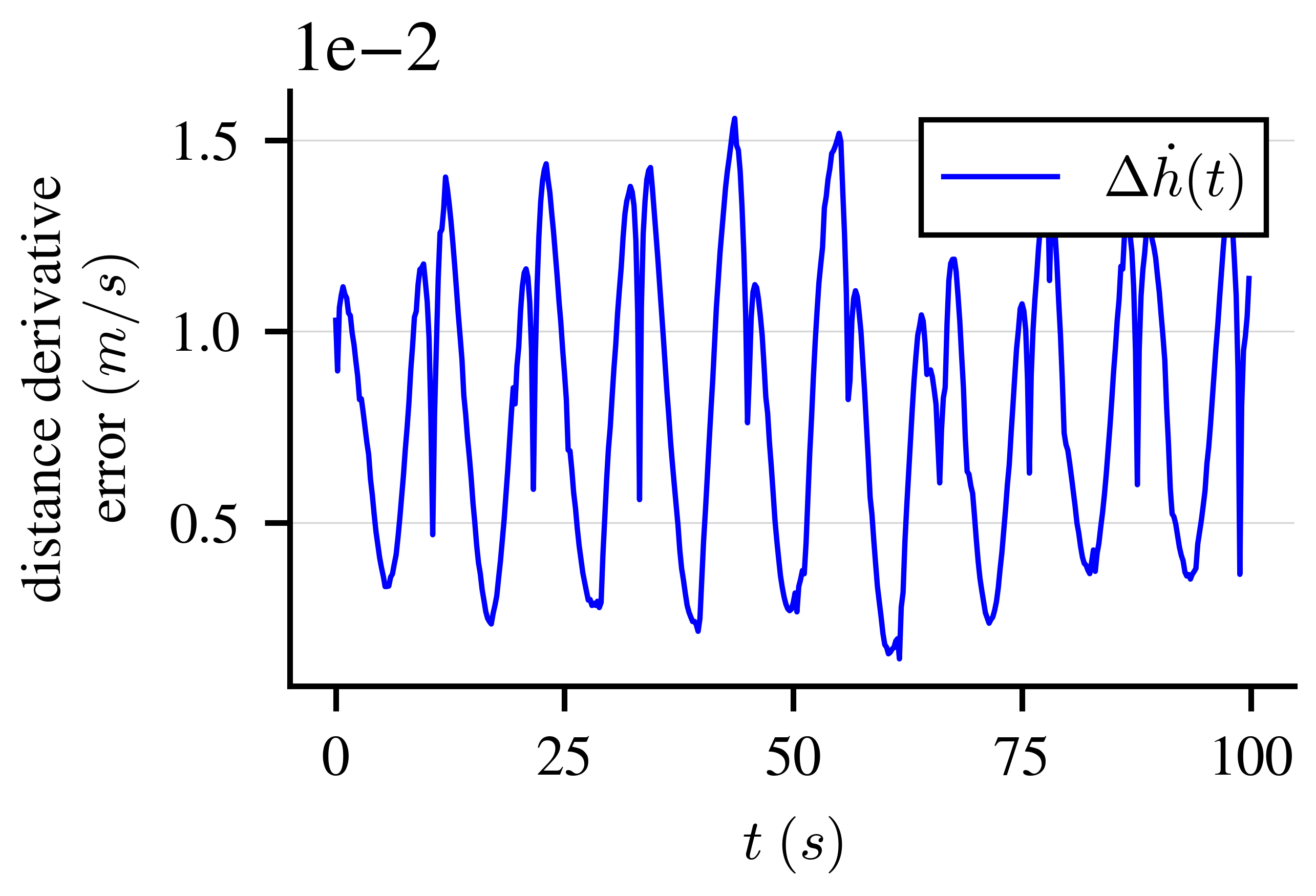}
    }%
    }%
    \hspace*{\fill}%
    \caption{
    Verification of the KKT solution ODE, \cref{thm:distance-ode}, and the derivative of the minimum distance, \cref{thm:minimum-distance-derivative}.
    \revision{
    The simulation environment, \cref{subfig:example-kkt-ode-env}, consists of a static polytope and a quadrotor system with a given reference trajectory.
    The safe region of the quadrotor comprises the quadrotor shape and a state-dependent position uncertainty set, as depicted in \cref{subfig:quadrotor-uncertainty-safe-set}.
    }
    The initial KKT solution is found by solving the minimum distance problem \cref{eq:strict-convex-min-dist}, and subsequently by integration of the KKT solution ODE, \cref{thm:distance-ode}.
    The solution obtained via the KKT solution ODE is compared to the actual KKT solution (obtained by solving \cref{eq:strict-convex-min-dist} at each timestep) in \cref{subfig:example-kkt-ode-kkt,subfig:example-kkt-ode-dist}.
    The minimum distance derivative from \cref{thm:minimum-distance-derivative} is compared to the actual derivative of the minimum distance (obtained via finite difference method) in \cref{subfig:example-kkt-ode-Ddist}.
    }
    \label{fig:example-kkt-ode}
\end{figure}

To verify the KKT solution ODE in \cref{thm:distance-ode}, we chose a predefined sinusoidal reference trajectory for the quadrotor around the obstacle, as shown in \cref{subfig:example-kkt-ode-env}.
The control inputs are then computed using a geometric controller similar to \cite{lee2010geometric}.
Given the control inputs $u$ and the current state $x$, we can compute the state derivative $\dot{x}$ and then use the KKT solution ODE to update the KKT solutions at the next time step.
}
Euler integration is used to integrate the KKT solution ODE in \cref{thm:distance-ode}, with a timestep of $\SI{1}{ms}$.
The actual KKT solution and minimum distance are obtained by solving \cref{eq:strict-convex-min-dist} at each timestep and are denoted by $(z^*, \lambda^*)$ and $h$, respectively.
The errors between the integrated solution from \cref{thm:distance-ode} and the actual solution are denoted by $(\Delta z^*, \Delta \lambda^*)$ and $\Delta h$, for the KKT solution and the minimum distance respectively.
The error between the derivative of the minimum distance obtained from \cref{thm:minimum-distance-derivative} and that from the actual KKT solution is denoted by $\Delta \dot{h}$.
\Cref{tab:example-kkt-ode,fig:example-kkt-ode} show the results from the simulation.
Note that p$50$ and p$99$ denote the $50$-th and $99$-th percentile values respectively.

\revision{
From \cref{tab:example-kkt-ode}, we can see that using the KKT solution ODE to obtain the KKT solution is roughly two orders of magnitude faster than solving the minimum distance problem \cref{eq:strict-convex-min-dist} (with warm start enabled).
As noted at the beginning of this section, when the degenerate active set $\mathcal{J}_2$ is nonempty, the QP \cref{eq:distance-ode-j2} is solved by first rewriting the problem in terms of the $\dot{\lambda}^*_{\mathcal{J}_2}$ variable, and subsequently using an LCP solver (implementation details can be found in the code).
On average, the transformed LCP problem can be solved within $\SI{0.5}{\mu s}$, while the transformation itself can be performed within $\SI{10}{\mu s}$.
Note that the KKT solution time in \cref{tab:example-kkt-ode} includes the time to solve the QP \cref{eq:distance-ode-j2} (whenever applicable).
}
Moreover, the relative KKT solution errors and the relative distance error are on the order of $10^{-3}$.
Similarly, the error in the distance derivative is on the order of $10^{-3} \, m/s$.
\revision{Both of these errors depend on the quadrotor speed, which can be gauged by the maximum distance derivative ($\SI{0.59}{m/s}$ in this case).}
The relative KKT solution errors, the relative distance error, and the distance derivative error are plotted in \cref{subfig:example-kkt-ode-kkt,subfig:example-kkt-ode-dist,subfig:example-kkt-ode-Ddist}.
\revision{To conclude, the KKT solution ODE provides a fast and accurate method to obtain the KKT solutions of the minimum distance problem \cref{eq:strict-convex-min-dist} along a state trajectory}.

\begin{table}[tbp]
\footnotesize
\caption{\revision{Simulation statistics for Example 1 in \cref{subsec:example-kkt-solution-ode-verification}.}}
\label{tab:example-kkt-ode}
\begin{center}
\begin{tabular}{|c|c|c||c|c|c|} \hline
\multicolumn{2}{|l|}{Statistic} & Value & \multicolumn{2}{l|}{Statistic} & Value \\ \hline
\multicolumn{1}{|l|}{\multirow{4}{*}{\multilinecell[l]{KKT ODE solution}{time ($\mu s$)}}} & mean & $7.77$ & \multicolumn{1}{l|}{\multirow{4}{*}{\multilinecell[l]{Distance derivative}{error ($m/s$), \\ $|\Delta \dot{h}|$}}} & mean & $8.09 \times 10^{-3}$ \\ \cline{2-3} \cline{5-6}
                  & std & $3.24$ & & std & $3.71 \times 10^{-3}$ \\ \cline{2-3} \cline{5-6}
                  & p50 & $7.00$ & & p50 & $8.26 \times 10^{-3}$ \\ \cline{2-3} \cline{5-6}
                  & p99 & $22.0$ & & p99 & $15.0 \times 10^{-3}$ \\ \hline
\multicolumn{2}{|l|}{\multilinecell[l]{Mean opt. solution}{time ($\mu s$)}} & $698$ & \multicolumn{2}{l|}{\multilinecell[l]{Max. distance derivative}{($m/s$), $\max \{|\dot{h}|\}$}} & $0.594$ \\[1pt] \hline
\multicolumn{2}{|l|}{\multilinecell[l]{Max. distance relative}{error, $| \Delta h| / h$}} & $2.90 \times 10^{-4}$  & \multicolumn{2}{l|}{\multilinecell[l]{Max. distance error ($m$),}{$\max\{|\Delta h|\}$}} & $6.10 \times 10^{-4}$ \\ \hline
\multicolumn{2}{|l|}{\multilinecell[l]{Max. primal solution}{relative error, $\lVert \Delta z^*\rVert / \lVert z^*\rVert$}} & $1.00 \times 10^{-3}$ & \multicolumn{2}{l|}{\multilinecell[l]{Max. dual solution}{relative error, $\lVert \Delta \lambda^*\rVert / \lVert \lambda^*\rVert$}} & $1.49 \times 10^{-3}$ \\
\hline
\end{tabular}
\end{center}
\end{table}

\subsection{Example 2: Obstacle avoidance using backup CBFs}
\label{subsec:example-cbf-obstacle avoidance}

\begin{figure}[tbp]
    \centering
    \subfloat[
    Snapshots of the quadrotor system as it navigates through a corridor filled with obstacles.
    ]{
    \label{subfig:example-cbf-env}
    {
    \includegraphics[width=4.8in]{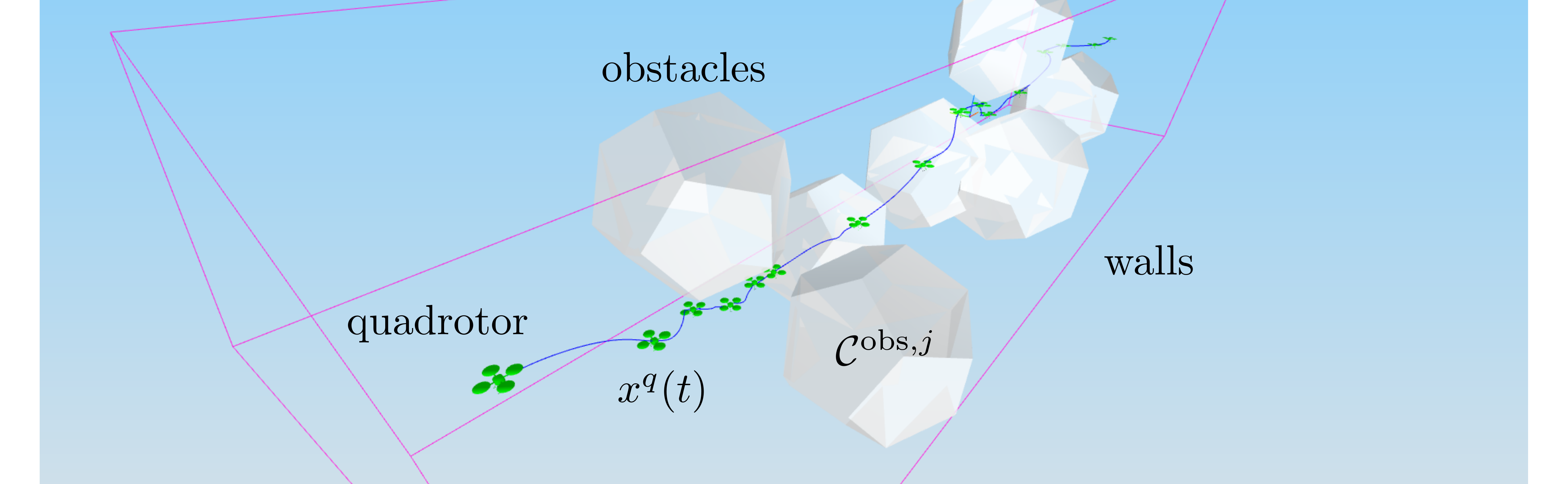}
    }%
    }%

    \hspace*{\fill}%
    \subfloat[
    Minimum distance across all collision pairs.
    ]{
    \label{subfig:example-cbf-dist}
    {
    \includegraphics[width=2.20in]{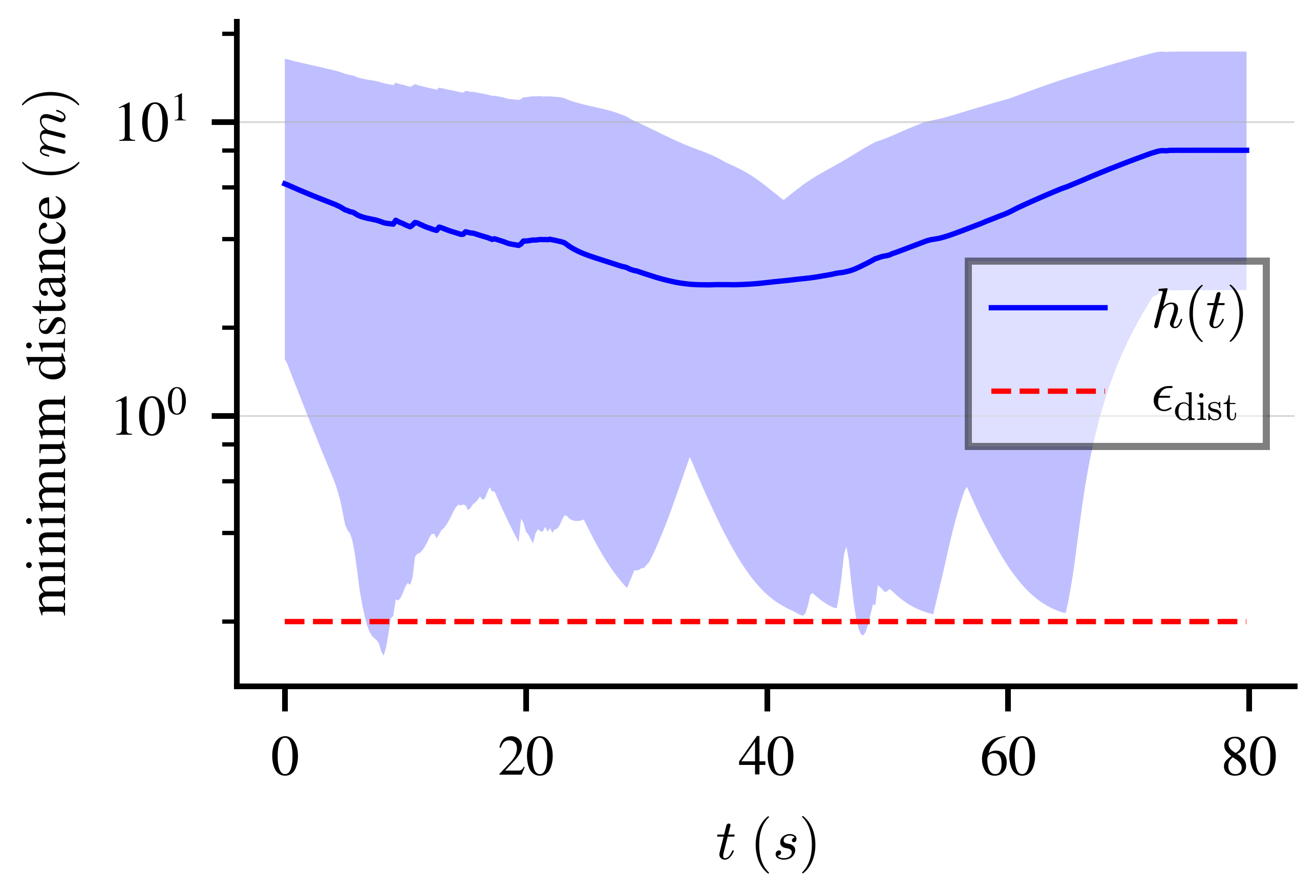}
    }%
    }%
    \centerhfill%
    \subfloat[
    Relative errors of the KKT solutions.
    ]{
    \label{subfig:example-cbf-kkt}
    {
    \includegraphics[width=2.20in]{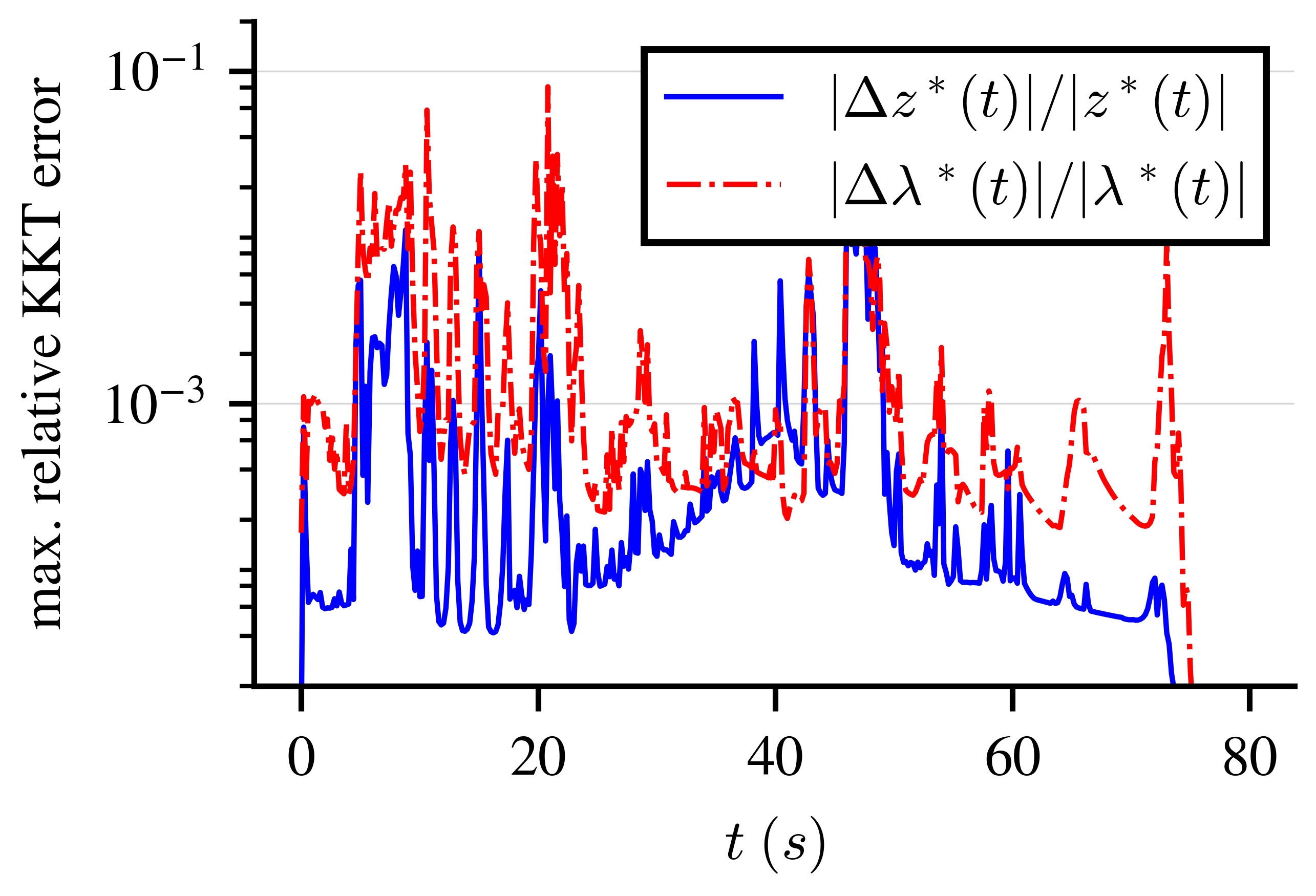}
    }%
    }%
    \hspace*{6pt}%

    \hspace*{\fill}%
    \subfloat[
    Relative errors of the minimum distance.
    ]{
    \label{subfig:example-cbf-rel-dist}
    {
    \includegraphics[width=2.20in]{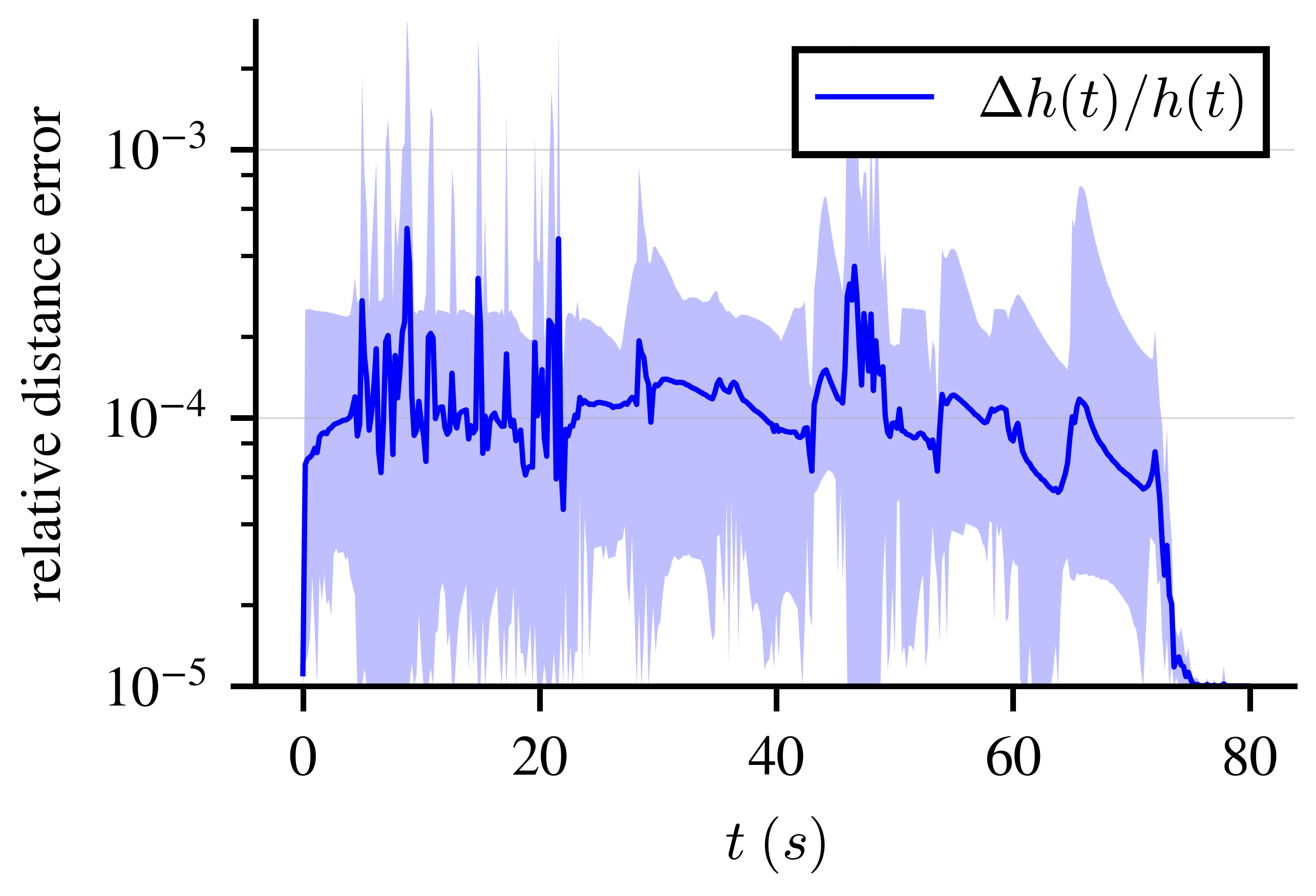}
    }%
    }%
    \hspace*{\fill}%
    \caption{
    \revision{
    Verification of the collision avoidance result \cref{thm:cbf-qp} for a quadrotor system navigating through an obstacle-filled corridor.
    The simulation environment, \cref{subfig:example-cbf-env}, consists of a quadrotor, $7$ polytopic obstacles, and $4$ walls.
    The control task is to navigate through the corridor while avoiding obstacles.
    The CBF-QP \cref{thm:cbf-qp} is used to guarantee the safety of the quadrotor system.
    In \cref{subfig:example-cbf-dist,subfig:example-cbf-rel-dist}, the colored regions show the range of a quantity across all $11$ collision pairs, while the solid lines show the mean values.
    Note that the $y$-axes of all the plots are in the log scale.
    Since, by \cref{subfig:example-cbf-dist}, the minimum distance across all collision pairs is greater than safety margin $\epsilon_\text{dist}$, the quadrotor system safely navigates through the corridor.
    }
    }
    \label{fig:example-cbf-obstacle-avoidance}
\end{figure}

In this subsection, we show how our method can be used for safety for higher-order systems, i.e., with relative degree $> 1$.
Recall that the CBF constraint in \eqref{eq:cbf-qp} uses the first derivative of $h$.
If the distance derivative does not depend on the system input $u$, then the CBF constraint does not depend on the system input and cannot be enforced.
One way to use CBFs for higher-order systems is using HoCBFs~\cite{xiao2019control}, which require higher-order derivatives of the minimum distance.
Such higher-order derivatives may not always exist for general strongly convex maps (although when $r^i = r^j = 1$, the minimum distance is smooth, and higher-order derivatives can be computed by differentiating \cref{eq:strict-convex-kkt-solution-derivative}).
However, verifying safety for systems with bounded inputs using HoCBFs is challenging.

\revision{In this section, we use backup CBFs~\cite{chen2021backup} to construct and enforce a CBF for the quadrotor system \cref{eq:quadrotor-dynamics}.}
In the method of backup CBFs, a backup controller is used to drive the system to a control invariant set contained within the safe set.
The minimum distance between the closed-loop trajectory from the backup controller (the backup trajectory) and the unsafe set is used as the CBF (the backup CBF).
Intuitively, instead of directly requiring the state to be inside the safe set, the backup CBF method requires that the entire backup trajectory (starting from the current state) lie in the safe set.

\revision{
Note that if the CBF $h$ for the quadrotor system depends only on the position $p$ and orientation $R$, then $\dot{h}$ does not depend on the thrust $T$.
In particular, such a CBF ignores the thrust-orientation dynamics of the quadrotor.
Following the idea of backup CBFs, we construct a CBF for the quadrotor system by extending the safe region of the quadrotor depending on its velocity and orientation.
For example, if the quadrotor has a large speed, the safe region is extended by a large amount along its velocity direction.
Likewise, if the orientation $R$ of the quadrotor is such that $Re_3$ aligns with the velocity $v$, then the safe region is extended by a large amount; this is based on the heuristic that $v^\top Re_3$ should be small in order to stop the quadrotor.

To define the safe region, we first design a braking corridor as follows (see \cite{wang2017safety}):
\begin{equation*}
\mathcal{C}^3(x) := \{z \in \mathbb{R}^3: \bar{z} = z - \xi/2, \ \bar{z}^\top (25 I - 24 \xi \xi^\top / \lVert \xi\rVert^2) \bar{z} \leq (0.1)^2 + \lVert \xi\rVert^2 / 4\},
\end{equation*}
where the braking distance vector $\xi$ at state $x$ is given by
\begin{equation*}
\xi = v (T_\text{brake} + T_\text{rot} (1 + v^\top Re_3 / \lVert v\rVert)),
\end{equation*}
where $T_\text{brake}$ and $T_\text{rot}$ are braking time constants.
Also, the norms of all the vectors in the definition of $\mathcal{C}^3$ are computed as $\lVert v\rVert = \sqrt{\epsilon^2 + v^\top v}$, where $\epsilon > 0$ is a small constant, to ensure the smoothness of the norm function.
Then, the map $\mathcal{C}^3$ is a strongly convex map (as defined in \cref{def:strongly-convex-set}).
The safe region for the quadrotor is defined as the Minkowski sum of the quadrotor shape set $\mathcal{C}^1(x)$ (defined in \cref{eq:example-quadrotor-shape-set}) and the braking corridor set $\mathcal{C}^3(x)$.
\Cref{subfig:quadrotor-corridor-safe-set} depicts the construction of this dynamic safe region for the quadrotor system.
}

\revision{
The simulation environment consists of $7$ obstacles in a corridor defined by $4$ walls, as shown in \cref{subfig:example-cbf-env}.
The control task for the quadrotor system is to navigate safely through the corridor.
We use a geometric tracking controller (see \cite{lee2010geometric}) to generate the nominal control inputs.
Note that the nominal control inputs do not guarantee the system's safety.
In addition to the collision avoidance CBF constraints, we bound the quadrotor velocity and roll-pitch angles using CBF constraints (the implementation details can be found in the code).
The CBF-QP used to compute the control input has $4$ variables with $11$ collision avoidance CBF constraints, $2$ state CBF constraints, and $8$ input bound constraints, for a total of $21$ constraints.
}
Notably, the CBF-QP would have the same size and sparsity structure if the minimum distance between the robots had an explicit form (such as for spheres).
Similar to \cref{subsec:example-kkt-solution-ode-verification}, Euler integration is used to integrate the KKT solution ODE in \cref{thm:distance-ode}, with a timestep of $\SI{1}{ms}$.
The actual solutions $(z^*, \lambda^*)$ and $h$ and the errors $(\Delta z^*, \Delta \lambda^*)$, and $\Delta h$ are computed similar to \cref{subsec:example-kkt-solution-ode-verification}.
\Cref{tab:example-cbf-obstacle-avoidance,fig:example-cbf-obstacle-avoidance} show the results from the simulation.
In \cref{tab:example-cbf-obstacle-avoidance}, the computation time for the KKT solution ODE is the time it takes to update all KKT solutions for the $11$ collision pairs.

\begin{table}[tbp]
\footnotesize
\caption{Simulation statistics for Example 2 in \cref{subsec:example-cbf-obstacle avoidance}.}
\label{tab:example-cbf-obstacle-avoidance}
\begin{center}
\begin{tabular}{|c|c|c||cc|c|} \hline
\multicolumn{2}{|l|}{Statistic} & Value & \multicolumn{2}{l|}{Statistic} & Value \\ \hline
\multicolumn{1}{|l|}{\multirow{4}{*}{\multilinecell[l]{KKT solution ODE}{time (11x) ($\mu s$)}}} & mean & $66.5$ & \multicolumn{1}{l}{\multirow{4}{*}{\multilinecell[l]{CBF-QP solution}{time ($ms$)}}} & \multicolumn{1}{|l|}{mean} & $0.147$ \\ \cline{2-3} \cline{5-6}
                  & std & $9.46$ & & \multicolumn{1}{|l|}{std} & $0.342$ \\ \cline{2-3} \cline{5-6}
                  & p50 & $64.0$ & & \multicolumn{1}{|l|}{p50} & $0.038$ \\ \cline{2-3} \cline{5-6}
                  & p99 & $115$ & & \multicolumn{1}{|l|}{p99} & $1.34$ \\ \hline
\multicolumn{2}{|l|}{\multilinecell[l]{Max. distance relative}{error, $| \Delta h| / h$}} & $6.56 \times 10^{-3}$  & \multicolumn{2}{l|}{\multilinecell[l]{Max. primal solution}{relative error, $\lVert \Delta z^*\rVert / \lVert z^*\rVert$}} & $0.0143$ \\ \hline
\multicolumn{2}{|l|}{\multilinecell[l]{Max. dual solution}{relative error, $\lVert \Delta \lambda^*\rVert / \lVert \lambda^*\rVert$}} & $0.0871$  & \multicolumn{2}{l|}{\multilinecell[l]{Fraction of optimization}{re-initializations}} & $0.0290$ \\
\hline
\end{tabular}
\end{center}
\end{table}

\revision{
From \cref{subfig:example-cbf-dist}, we can see that the minimum distance between all safe regions is always above $\epsilon_\text{dist}$, a positive safety margin, for the entire trajectory (the minimum distance can dip below the safety margin slightly due to numerical errors).
Thus, the CBF-QP guarantees that the closed-loop quadrotor system remains safe.
From \cref{tab:example-cbf-obstacle-avoidance}, we can see that the KKT solution ODE can update the KKT solutions for all $11$ collision pairs reliably and accurately at $\gg \SI{1000}{Hz}$.
However, for $2.90\%$ of KKT solution updates, we need to re-initialize the minimum distance optimization problem for at least one collision pair due to inaccurate KKT solution updates (the time it takes to check the KKT error is included in the KKT solution ODE time in \cref{tab:example-cbf-obstacle-avoidance}).
We note that such re-initializations are needed due to the highly dynamic nature of the quadrotor safe region, and thus were not needed for the example in \cref{subsec:example-kkt-solution-ode-verification}.
\revision{The larger errors in the KKT solution (compared to the example in \cref{subsec:example-kkt-solution-ode-verification}) can be attributed to large distance derivative values; $8.65 \, m/s$ for this example as compared to $0.59 \, m/s$ for the previous example.}
Even for such dynamic scenarios, we can use the KKT solution ODE at $\SI{1000}{Hz}$.
In contrast, the direct solution time for the minimum distance problem is $\SI{7.17}{ms}$ (for all $11$ collision pairs).
Finally, we note that the CBF-QP can be solved reliably at $\SI{500}{Hz}$.
Importantly, there is no control frequency penalty for solving the CBF-QP using our method compared to a case where the minimum distance can be calculated explicitly.
This justifies our claim that using the exact convex shape of a safe region (which can be represented by a strongly convex map) only has a minor performance penalty compared to using a convex overapproximation.
To conclude, our method allows for real-time collision avoidance between strongly convex maps.
}

\section{Conclusion}
\label{sec:conclusion}

In this paper, we presented a general framework for \revision{collision} avoidance between strongly convex maps using the KKT solution of the minimum distance optimization problem.
We showed that the time-varying KKT solution of the minimum distance problem between a strongly convex pair (along a state trajectory) can be computed using an ODE.
Moreover, we showed the explicit form of the minimum distance derivative and that collision avoidance between a strongly convex pair can be achieved using a CBF-QP for systems with control affine dynamics, enabling real-time implementation.
We validated our results on the KKT solution ODE and the CBF-QP formulation with an obstacle avoidance task for a \revision{quadrotor} system with strongly convex-shaped robots.
\revision{
Our method enables real-time collision avoidance using CBFs for state-dependent convex sets (without overapproximation), allows for convex set algebraic operations, and is compatible with (and enhances) other methods such as backup CBFs, distributed CBFs, robust CBFs, and higher-order CBFs.
Finally, we proved strong safety for the CBF-QP (using discontinuous dynamical systems theory) with minimal assumptions on the CBF-QP structure.
}





\appendix

\section{Proof of \crefbookmark{Lemma}{lem:strict-convex-lipschitz-continuity}}%
\label{app:proof-strict-convex-lipschitz-continuity}%

\revision{From \cref{assum:strongly-convex-pair}, at least one of $\mathcal{C}^i$ or $\mathcal{C}^j$ is a strongly convex map; let $\mathcal{C}^i$ be a strongly convex map without loss of generality.}

\subsection{Uniqueness of optimal solution}
\label{app:proof-strict-convex-lipschitz-continuity-uniqueness}%

\revision{Let $\mathcal{W} \supset \mathcal{C}^i(x^i)$ be a compact set.}
By \cref{def:strongly-convex-set}, $A^i_k(x^i, \cdot)$ is strongly convex \revision{on $\mathcal{W}$ for all $k \in [r^i]$}.
Thus, $\mathcal{C}^i(x^i)$, \revision{which is the $0$-sublevel set of $\max_k A^i_k(x^i, \cdot)$,} is a strictly convex set, i.e., $\mu z^{i1} + (1-\mu) z^{i2} \in \text{Int} (\mathcal{C}^i(x^i))$ for all $\mu \in (0, 1)$, and $z^{i1}, z^{i2} \in \mathcal{C}^i(x^i)$.

\revision{
Now, we prove the uniqueness of the optimal solution of \cref{eq:strict-convex-min-dist} using the fact that $\mathcal{C}^i(x^i)$ is a strictly convex set.
}
The optimization problem \cref{eq:strict-convex-min-dist} is equivalent to
\begin{equation}
\label{eq:convex-set-min-dist-reduced}
h(x) = \min_{s} \ \ \{\lVert s \rVert_2^2 : \ s \in \mathcal{C}^i(x^i) -\mathcal{C}^j(x^j)\},
\end{equation}
where $\mathcal{C}^i(x^i) - \mathcal{C}^j(x^j) := \{s = z^i - z^j: z^i \in \mathcal{C}^i(x^i), z^j \in \mathcal{C}^j(x^j)\}$ is the Minkowski sum of $\mathcal{C}^i(x^i)$ and $-\mathcal{C}^j(x^j)$.
Since the cost function in \cref{eq:convex-set-min-dist-reduced} is strictly convex, there is a unique optimal solution $s^*(x)$ of \cref{eq:convex-set-min-dist-reduced}.
Then, since $\mathcal{C}^i(x^i)$ is strictly convex and $h(x) > 0$, there is a unique optimal solution of \cref{eq:strict-convex-min-dist} at $x$ satisfying $s^*(x) = z^{i*}(x) - z^{j*}(x)$.

\subsection{Continuity of optimal solution}
\label{app:proof-strict-convex-lipschitz-continuity-continuity}%

\revision{To show the continuity of the optimal solution $z^*(x)$, we will show that $\mathcal{C}^i$ and $\mathcal{C}^j$ are \emph{continuous convex maps}.}
\begin{definition}[Continuous convex map]%
\label{def:continuous-convex-set}%
A set-valued map $\mathcal{C}^i$ is a continuous convex map if:
\begin{enumerate}[ref={\thedefinition.\arabic*}, leftmargin=*]
    \item \label[definition]{subdef:continuous-convex-set-compact-regular}
    The set $\mathcal{C}^i(x^i)$ is convex, compact, and has a non-empty interior $\forall x^i \in \mathcal{X}^i$.
    
    \item \label[definition]{subdef:continuous-convex-set-continuous}
    The set-valued map $\mathcal{C}^i$ is continuous, i.e., both lower and upper semi-continuous.
\end{enumerate}

\end{definition}
Since \cref{assum:strongly-convex-pair} holds, $\mathcal{C}^i$ and $\mathcal{C}^j$ are smooth convex maps.
So, $\mathcal{C}^i$ and $\mathcal{C}^j$ satisfy \cref{subdef:continuous-convex-set-compact-regular} (by \cref{def:smooth-convex-set}).
Also, by \cref{def:smooth-convex-set}, $A^i$ and $A^j$ are continuous and have convex component functions, and \revision{$\mathcal{C}^i(x^i)$ and $\mathcal{C}^j(x^j)$ have non-empty interiors for all $x \in \mathcal{X}$.
So, $\mathcal{C}^i$ and $\mathcal{C}^j$ are upper and lower semi-continuous set-valued maps~\cite[Ex.~5.10]{rockafellar2009variational}}, and thus are continuous convex maps.

\revision{
Then, the continuity of the minimum distance function $h$ can be shown using the fact that $\mathcal{C}^i$ and $\mathcal{C}^j$ are continuous convex maps and using the maximum theorem~\cite[Ch.~VI]{berge1963topological} (also see \cite[Thm.~1.17]{rockafellar2009variational}) on \cref{eq:strict-convex-min-dist}.
Since $h(x) > 0$ by assumption, we can find a neighborhood $\mathcal{N}(x)$ of $x$ such that $h(x') > 0 \ \forall x' \in \mathcal{N}(x)$.
Then, by \cref{app:proof-strict-convex-lipschitz-continuity-uniqueness}, there is unique optimal solution of \cref{eq:strict-convex-min-dist} for all $x' \in \mathcal{N}(x)$.
Applying the maximum theorem on $\mathcal{N}(x)$, we get that the optimal solution $z^*(\cdot)$ is continuous on $\mathcal{N}(x)$.
}






\subsection{Local Lipschitz continuity of the minimum distance function}
\label{app:proof-strict-convex-lipschitz-continuity-lipschitz-continuity}%

\revision{We show the local Lipschitz continuity of $h$ using \cite[Thm.~1]{klatte1985stability}.}
Let $\mathcal{W} \supset \mathcal{C}^i(x^i) \times \mathcal{C}^j(x^j)$ be an open, convex, and bounded set \revision{(this is possible because $\mathcal{C}^i(x^i)$ and $\mathcal{C}^j(x^j)$ are compact)}.
We first note that the objective $\lVert z^i - z^j\rVert_2^2$ is \revision{globally} Lipschitz continuous \revision{in $(x, z)$} on the set $\mathcal{X} \times \mathcal{W}$, since $\mathcal{W}$ is bounded.
By \cref{def:smooth-convex-set}, \revision{and the boundedness of $\mathcal{W}$, $\max_{z \in \mathcal{W}} \lVert \nabla_{x^i} A^i_k(x^i, z^i)\rVert_2$ is bounded for all $k \in [r^i]$ (and likewise for $A^j_k$).}
So, $A^i_{k^i}(\cdot, z^i)$ and $A^j_{k^j}(\cdot, z^j)$ are locally Lipschitz continuous at $x$ for each $z \in \mathcal{W}$, $k^i \in [r^i]$, and $k^j \in [r^j]$, with a Lipschitz constant independent of $z \in \mathcal{W}$.
Then, by~\cite[Thm.~1]{klatte1985stability}, $h$ is locally Lipschitz continuous at $x$.
\revision{Since this holds at each $x \in \mathcal{X}$ (even if $h(x) = 0$), $h$ is a locally Lipschitz continuous function.}

\revision{
We note that the above proof does not require \cref{assum:strongly-convex-pair}, but only that $\mathcal{C}^i$ and $\mathcal{C}^j$ are smooth convex maps.
Additionally, the vector functions $A^i$ and $A^j$ only need to be once continuously differentiable.
}

\section{Proof of \crefbookmark{Lemma}{lem:strict-convex-ssosc}}%
\label{app:proof-strict-convex-ssosc}%
The KKT conditions are necessary and sufficient for the optimality of \cref{eq:strict-convex-min-dist}, under \cref{def:smooth-convex-set}.
\revision{Since $h(x) > 0$, by \cref{lem:strict-convex-lipschitz-continuity}, there is a unique primal optimal solution at $x$, denoted by $z^*(x)$}.
So, there is at least one KKT solution $(z^*(x), \lambda^*)$ for \cref{eq:strict-convex-min-dist} which satisfies the KKT conditions \cref{eq:strict-convex-kkt}.
In particular, the KKT solution satisfies the condition \cref{subeq:strict-convex-kkt-gradient}, which can be expanded using \cref{eq:strict-convex-lagrangian} as
\begin{flalign}
\label{eq:strict-convex-kkt-gradient-expanded}
&& \underbrace{\begin{bmatrix} [D_{z^i} A^{i}(x^i,z^{i*}(x))]^\top & 0_{l\times r^j} \\ 0_{l\times r^i} & [D_{z^j} A^{j}(x^j,z^{j*}(x))]^\top\end{bmatrix}}_{= [D_z A(x,z^*(x))]^\top}\begin{bmatrix}\lambda^{i*} \\ \lambda^{j*}\end{bmatrix} = \revision{\begin{bmatrix}-2(z^{i*}(x) - z^{j*}(x)) \\ 2(z^{i*}(x) - z^{j*}(x))\end{bmatrix}.}
\end{flalign}
\revision{Since $h(x) > 0$ by assumption, $z^{i*}(x) \neq z^{j*}(x)$}.
From \cref{eq:strict-convex-kkt-gradient-expanded}, since the RHS is not zero, $\lambda^{i*}_{k^i} > 0$ and $\lambda^{j*}_{k^j} > 0$ for some $k^i \in [r^i]$ and $k^j \in [r^j]$.
The $z$-Hessian of the Lagrangian $L$ is given by
\begin{equation*}
\label{eq:strict-convex-kkt-hessian-expanded}
\nabla_z^2 L(x, z^*(x), \lambda^*) = \begin{bmatrix}2I + \sum_{m=1}^{r^i} \, \lambda^{i*}_m \nabla_{z^i}^2 A^i_m & -2I \\ -2I & 2I + \sum_{n=1}^{r^j} \, \lambda^{j*}_{n} \nabla^2_{z^j} A^j_n \end{bmatrix}(x, z^*(x)).
\end{equation*}
\revision{
By \cref{assum:strongly-convex-pair}, $(\mathcal{C}^i, \mathcal{C}^j)$ is a strongly convex pair, i.e., at least one of $\mathcal{C}^i$ or $\mathcal{C}^j$ is a strongly convex map.
Without loss of generality, let $\mathcal{C}^i$ be a strongly convex map.
}
Using $\lambda^* \geq 0$ and $\nabla_{z^i}^2 A^i_m(x^i, z^{i*}(x)) \succ 0$ (by \cref{def:strongly-convex-set}), we get that
\begin{equation*}
\textstyle{\sum}_{m=1}^{r^i} \, \lambda^{i*}_m \nabla_{z^i}^2 A^i_m(x^i, z^{i*}(x)) \succeq \lambda^{i*}_{k^i} \nabla_{z^i}^2 A^i_{k^i}(x^i, z^{i*}(x)) \succ 0.
\end{equation*}

Thus, $\nabla_z^2 L(x, z^*(x), \lambda^*) \succ 0$, meaning \cref{eq:strict-convex-ssosc} is satisfied, and SSOSC holds at $x$.

\section{Proof of \crefbookmark{Theorem}{thm:distance-ode}}%
\label{app:proof-distance-ode}%
\revision{Given the assumptions of \cref{thm:distance-ode}, \cref{lem:strict-convex-unique-continuous-kkt-solution} shows that the KKT solution is unique for all $t \geq t_0$.}
Consider the cases when \revision{the degenerate active set of constraints} $\mathcal{J}_2(x(t)) = \emptyset$ and $\mathcal{J}_2(x(t)) \neq \emptyset$ \revision{(see \cref{eq:strict-convex-kkt-active-set})}.
For the first case, \cref{lem:strict-convex-kkt-solution-c1} provides the derivative of the KKT solution, and for the second case, \cref{prop:strict-convex-kkt-solution-directional-derivative} provides the directional derivative.
To use these derivatives to integrate the KKT solution numerically, we add stabilizing terms that ensure convergence of the numerical KKT solution in the locality of the actual KKT solution.
The KKT solution is the pair $(z, \lambda)$ that uniquely satisfies the KKT conditions \cref{eq:strict-convex-kkt}.
Thus, we add the residual error in the KKT conditions to the ODE for both cases.

\begin{enumerate}[labelindent=0pt, itemindent=!, leftmargin=*]
\item Case 1: $\mathcal{J}_2(x(t)) = \emptyset$:
When $\mathcal{J}_2(x(t)) = \emptyset$, strict complementary slackness holds.
By \cref{lem:strict-convex-kkt-solution-c1}, strict complementary slackness holds in a neighborhood of $x(t)$.
Thus, to satisfy the KKT conditions \cref{eq:strict-convex-kkt}, only \cref{subeq:strict-convex-kkt-gradient} and \cref{subeq:strict-convex-kkt-slackness} need to be satisfied, i.e., $e_{kkt}(x(t))$ \revision{(see \cref{eq:distance-ode-e-kkt})} should be zero.
Moreover, by \cref{lem:strict-convex-kkt-solution-c1}, \cref{eq:strict-convex-kkt-solution-derivative} gives the derivative of the KKT solution \revision{as a function of the state.
By the differentiability of $x(t)$, we get that the KKT solutions are differentiable as a function of time at $t$}.
Adding $e_{kkt}(x(t))$ as a stabilizing term to \cref{eq:strict-convex-kkt-solution-derivative}, we get,
\begin{equation*}
    Q(x(t))\begin{bmatrix}\dot{z}^*(t) \\ \dot{\lambda}^*(t)\end{bmatrix} = V(x(t)) \dot{x}(t) - \kappa e_{kkt}(x(t)),
\end{equation*}
which is the same as \cref{eq:distance-ode-empty-j2}.

\item Case 2: $\mathcal{J}_2(x(t)) \neq \emptyset$:
In this case, \cref{eq:strict-convex-kkt-right-derivative} provides the \revision{directional derivative of the KKT solution along $\mathring{x}$.
Since $x(t)$ is differentiable, the KKT solutions are right-differentiable as a function of time at $t$ with $\mathring{x} = \dot{x}(t)$}.
Similar to case $1$, we can add stabilizing terms to all the equality constraints in \cref{eq:strict-convex-kkt-right-derivative} to get,
\begin{subequations}
\label{eq:app-strict-convex-kkt-right-derivative}%
\begin{alignat}{2}
& \nabla_z^2L \mathring{z} + D_z A^\top \mathring{\lambda} = -D_x (\nabla_z L)[\dot{x}(t)] - \kappa \nabla_z L &&, \label{subeq:app-strict-convex-kkt-right-derivative-gradient} \\
& D_z A_{\mathcal{J}_1} \mathring{z} = -D_x A_{\mathcal{J}_1} \dot{x}(t) - \kappa A_{\mathcal{J}_1}, \label{subeq:app-strict-convex-kkt-right-derivative-feasibility} \\
& D_z A_{\mathcal{J}_2} \mathring{z} \leq -D_x A_{\mathcal{J}_2} \dot{x}(t), \\
& \mathring{\lambda}_{\mathcal{J}_{0c}} = -\kappa \lambda^*_{\mathcal{J}_{0c}}, \quad \mathring{\lambda}_{\mathcal{J}_2} \geq 0, \label{subeq:app-strict-convex-kkt-right-derivative-non-neg}\\
& \mathring{\lambda}_{\mathcal{J}_2}^\top (D_z A_{\mathcal{J}_2} \mathring{z} + D_x A_{\mathcal{J}_2} \dot{x}(t)) = 0, \label{subeq:app-strict-convex-kkt-right-derivative-slackness}%
\end{alignat}%
\end{subequations}%
where all terms are evaluated at $(x(t), z^*(t), \lambda^*(t))$, \revision{and the strictly active set $\mathcal{J}_1$ and the degenerate active set $\mathcal{J}_2$ are defined in \cref{eq:strict-convex-kkt-active-set}.
$\mathcal{J}_{0c}$ (the inactive set) is the complement of the active set $\mathcal{J}_0$}.
Note that,
\begin{equation*}
    (D_z A)^\top \mathring{\lambda} = (D_z A_{\mathcal{J}_{0c}})^\top \mathring{\lambda}_{\mathcal{J}_{0c}} + (D_z A_{\mathcal{J}_{1}})^\top \mathring{\lambda}_{\mathcal{J}_{1}} + (D_z A_{\mathcal{J}_{2}})^\top \mathring{\lambda}_{\mathcal{J}_{2}}.
\end{equation*}
Substituting the above value in \cref{subeq:app-strict-convex-kkt-right-derivative-gradient}, and using the value of $\mathring{\lambda}_{\mathcal{J}_{0c}}$ from \cref{subeq:app-strict-convex-kkt-right-derivative-non-neg}, we get that
\begin{align*}
    \nabla_z^2L \mathring{z} + D_z A_{\mathcal{J}_{1}}^\top \mathring{\lambda}_{\mathcal{J}_{1}} + D_z A_{\mathcal{J}_{2}}^\top \mathring{\lambda}_{\mathcal{J}_{2}} = & \ -D_x (\nabla_z L)[\dot{x}(t)] - \kappa \nabla_z L + \kappa D_z A_{\mathcal{J}_{0c}}^\top \lambda^*_{\mathcal{J}_{0c}}, \\
    = & \ -V_1 \dot{x}(t) - \tilde{e}_{1, kkt}.
\end{align*}
Rewriting \cref{eq:app-strict-convex-kkt-right-derivative} using the constants in \cref{eq:distance-ode-j2-consts} and the above equation, we obtain
\begin{subequations}
\label{eq:app-strict-convex-kkt-right-derivative-kkt}%
\begin{alignat}{2}
& Q_{11} \mathring{z} + P_{eq}^\top \mathring{\lambda}_{\mathcal{J}_{1}} + P_{in}^\top \mathring{\lambda}_{\mathcal{J}_{2}} = -V_1 \dot{x}(t) - \tilde{e}_{1, kkt},\\
& P_{eq} \mathring{z} = q_{eq} \dot{x}(t) + \tilde{e}_{2, kkt}, \\
& P_{in} \mathring{z} \leq q_{in} \dot{x}(t), \quad \mathring{\lambda}_{\mathcal{J}_2} \geq 0, \quad \mathring{\lambda}_{\mathcal{J}_2}^\top (P_{in} \mathring{z} - q_{in} \dot{x}(t)) = 0.%
\end{alignat}%
\end{subequations}
The above constraints are the KKT conditions of the QP \cref{eq:distance-ode-j2}, with $(\mathring{\lambda}_{\mathcal{J}_1}, \mathring{\lambda}_{\mathcal{J}_2}) = (\lambda^*_{eq}, \lambda^*_{in})$.
By \cref{lem:strict-convex-ssosc}, we have that $Q_{11} \succ 0$, and thus the unique optimal solution of the QP \cref{eq:distance-ode-j2} is the \revision{right}-derivative of the KKT solution \revision{at time $t$}.
\revision{We also note that the variables $\mathring{z}$ and $\mathring{\lambda}_{\mathcal{J}_1}$ can be eliminated from \cref{eq:app-strict-convex-kkt-right-derivative-kkt} to obtain a reduced LCP in the variable $\mathring{\lambda}_{\mathcal{J}_2}$.}

\end{enumerate}

\section{Proof of \crefbookmark{Theorem}{thm:minimum-distance-derivative}}%
\label{app:proof-minimum-distance-derivative}%
Since the assumptions of \cref{prop:strict-convex-kkt-solution-directional-derivative} are satisfied, the KKT solution is directionally differentiable at $x$ for any direction $\mathring{x}$, with the derivatives $(\mathring{z}, \mathring{\lambda})$ solving \cref{eq:strict-convex-kkt-right-derivative}.
Since the distance between two points $d(z) = \lVert z^i - z^j \rVert^2_2$ is a smooth function \revision{of $z$}, the minimum distance $h(x) = d(z^*(x))$ is directionally differentiable, with
\begin{equation*}
    D_x h(x)[\mathring{x}] = \nabla_z d(z^*(x))^\top \mathring{z}.
\end{equation*}
Since $(z^*(x), \lambda^*(x))$ satisfy the KKT conditions, we have, from \cref{subeq:strict-convex-kkt-gradient}, that
\begin{align*}
    \nabla_z d(z^*(x)) \ & = \revision{ \begin{bmatrix}
        2(z^{i*}(x) - z^{j*}(x)) \\
        -2(z^{i*}(x) - z^{j*}(x))
    \end{bmatrix}} = -D_z A(x, z^*(x))^\top {\lambda^*}(x) \\
    & = -D_z A_{\mathcal{J}_1}(x, z^*(x))^\top {\lambda_{\mathcal{J}_1}^*}(x),
\end{align*}
since $\lambda_{\mathcal{J}_2}^*(x) = 0$ and $\lambda_{\mathcal{J}_{0c}}^*(x) = 0$, by definition.
\revision{Recall that $\mathcal{J}_2$ is the degenerate active set of constraints and $\mathcal{J}_{0c}$ (the inactive set) is the complement of the active set of constraints $\mathcal{J}_0$ (see \cref{eq:strict-convex-kkt-active-set})}.
Combining the above two equations and \cref{subeq:strict-convex-kkt-right-derivative-feasibility}, we get,
\begin{align*}
    D_x h(x)[\mathring{x}] = & \ \nabla_z d(z^*(x))^\top \mathring{z} = -\lambda_{\mathcal{J}_1}^*(x)^\top D_z A_{\mathcal{J}_1}(x, z^*(x)) \mathring{z}, \\
    = & \ \lambda_{\mathcal{J}_1}^*(x)^\top D_x A_{\mathcal{J}_1}(x, z^*(x)) \mathring{x} = \lambda^*(x) D_x A(x, z^*(x)) \mathring{x}.
\end{align*}
Since the directional derivative $D_x h(x)[\mathring{x}]$ is a linear function of $\mathring{x}$, $h$ is differentiable at $x$.
By \cref{subdef:smooth-convex-set-c2}, $A$ is twice continuously differentiable, and by \cref{lem:strict-convex-unique-continuous-kkt-solution} \revision{$(z^*(\cdot), \lambda^*(\cdot))$ is a continuous function of $x$ in a neighborhood $\mathcal{N}(x)$ of $x$}.
Thus, $h$ is continuously differentiable \revision{on $\mathcal{N}(x)$}.
Finally, \revision{by \cref{prop:strict-convex-kkt-solution-directional-derivative}}, $(z^*(\cdot), \lambda^*(\cdot))$ is directionally differentiable at $x$, and so $D_x h$ is directionally differentiable at $x$.

Now let the functions $A^i$ and $A^j$ defining $\mathcal{C}^i$ and $\mathcal{C}^j$ respectively be smooth and $r^i = r^j = 1$.
Then, it is always true that \revision{the degenerate active set} $\mathcal{J}_2(x) = \emptyset$, as long as $h(x) > 0$.
\revision{Let $\mathcal{N}(x)$ be such that $h(x') > 0 \ \forall x' \in \mathcal{N}(x)$.}
In this case, the derivative of the KKT solution is always given by \cref{eq:strict-convex-kkt-solution-derivative}.
Since the cost and constraints of the minimum distance problem are smooth, we have, by the smoothness of $Q(x)$ and $V(x)$, that the KKT solutions are a smooth function \revision{of $x$}, and therefore that the minimum distance is smooth \revision{on $\mathcal{N}(x)$}.

\section{Proof of \crefbookmark{Theorem}{thm:cbf-qp}}%
\label{app:proof-cbf-qp}%

The following two properties of the Filippov operator will be used to prove \cref{thm:cbf-qp}.

\begin{property}[Distributive property]%
\label{property:distributive-filippov-operator}%
By the sum and product rules for the Filippov operator \cite[Eq.~(26-27)]{cortes2008discontinuous}, and continuity of $f^i$ and $g^i$, \cref{eq:filippov-operator-def} is equivalent to
\begin{equation}
\label{eq:filippov-simplification}
\dot{x}^i(t) \in f^i(x^i(t)) + g^i(x^i(t)) F[u^i_{fb}](x(t)).
\end{equation}

\end{property}

\begin{property}[Subset property]%
\label{property:invariance-filippov-operator}%
Let $u_{fb}: \mathcal{X} \rightarrow \mathcal{U}$, $\mathcal{U} \subseteq \mathbb{R}^m$, be a map such that for all $x$, $u_{fb}(x) \in \mathcal{F}_u(x)$, where $\mathcal{F}_u: \mathcal{X} \rightarrow 2^{\mathbb{R}^m}$ is an upper semi-continuous, pointwise non-empty, convex, and compact set-valued map.
Then, 
$F[u_{fb}](x) \subseteq \mathcal{F}_u(x)$, $\forall x \in \mathcal{X}$.

\end{property}

Let $u^*_{fb}$ be a measurable control law obtained as the solution of \cref{eq:cbf-qp}.
\revision{$u^*_{fb}$ need not be a Lipschitz continuous (or continuous) function, and so we consider the Filippov solutions of the closed-loop system}.
By \cref{prop:existence-filippov-solution}, given $x_0$, there exists a Filippov solution $x: [t_0, T] \rightarrow \mathcal{X}$ for \revision{the closed-loop dynamical system \cref{eq:n-affine-systems-feedback} (with $u^*_{fb}$ as the feedback control law) for some $T > t_0$ with $x(t_0) = x_0$}.
\revision{
By \cref{def:caratheodory-solution,def:filippov-solution}, $x$ satisfies, for almost all $t \in [t_0, T]$,
\begin{equation}
\label{eq:cbf-qp-proof-xdot}
    \dot{x}(t) \in F[f(x(t)) + g(x(t)) u^*_{fb}](x(t)),
\end{equation}
where $F$ is the Filippov operator (see \cref{eq:filippov-operator-def}) and $f$ and $g$ are defined as $f(x) := (f^i(x^i), f^j(x^j))$ and $g(x) := (g^i(x^i), g^j(x^j))$.
Let $t$ be any such time where \cref{eq:cbf-qp-proof-xdot} is satisfied.
By \cref{property:distributive-filippov-operator} and the continuity of $f$ and $g$, we can write \cref{eq:cbf-qp-proof-xdot} as
\begin{equation*}
\begin{split}
    \dot{x}(t) & \in F[f(x(t)) + g(x(t)) u^*_{fb}](x(t)) = f(x(t)) + g(x(t)) F[u^*_{fb}](x(t)).
\end{split}
\end{equation*}
So, for almost all $t \in [t_0, T]$, $\exists \tilde{u}(t) \in F[u^*_{fb}](x(t))$ such that
\begin{equation}
\label{eq:cbf-qp-proof-u-tilde}
    \dot{x}(t) = f(x(t)) + g(x(t)) \tilde{u}(t).
\end{equation}
Note that $\tilde{u}(t)$ can be different from $u^*_{fb}(x(t))$.
Next, we will show that $\tilde{u}(t)$ satisfies the CBF constraint \cref{subeq:cbf-qp-cbf} at $x(t)$.
This is true even if $\tilde{u}(t)$ is not equal to $u^*_{fb}(x(t))$; the latter satisfies the CBF constraint by definition (see \cref{eq:cbf-qp}).
}

Let $\mathcal{F}_u(x) \subseteq \mathcal{U}$ be the feasible set of inputs of the CBF-QP \cref{eq:cbf-qp} at $x$.
By continuity of $f$, $g$, $h$ (\cref{lem:strict-convex-lipschitz-continuity}), $(z^*,\lambda^*)$ (\cref{lem:strict-convex-unique-continuous-kkt-solution}), and $D_x A$ (\cref{subdef:smooth-convex-set-c2}), all the constraints in \cref{eq:cbf-qp} are continuous in $x$, and thus $\mathcal{F}_u$ is an upper semi-continuous set-valued map~\cite[Ex.~5.8]{rockafellar2009variational}.
By compactness of $\mathcal{U}$, $\mathcal{F}_u$ is pointwise compact.
Since \cref{eq:cbf-qp} is a QP, $\mathcal{F}_u$ is pointwise convex.
Lastly, by the assumption on the feasibility of \cref{eq:cbf-qp}, $\mathcal{F}_u$ is pointwise non-empty \revision{(even when $h(x) = 0$)}.
\revision{
By definition of $u^*_{fb}$, $u^*_{fb}(x) \in \mathcal{F}_u(x)$.
Thus, by \cref{property:invariance-filippov-operator}, we have that $F[u^*_{fb}](x) \subset \mathcal{F}_u(x)$ for all $x$.

Combining the above two conclusions, we get that for almost all $t \in [t_0, T]$, $\exists \tilde{u}(t) \in F[u^*_{fb}](x(t)) \subset \mathcal{F}_u(x(t))$.
Let $t$ be any time where the closed-loop trajectory $x(t)$ is differentiable.
Then, since $h$ is continuously differentiable (\cref{thm:minimum-distance-derivative}), $h(x(t))$ is differentiable at $t$.
}
Since $\tilde{u}(t) \in \mathcal{F}_u(x(t))$, by \cref{subeq:cbf-qp-cbf}, we have that,
\begin{equation*}
\begin{split}
    \dot{h}(x(t)) & = \lambda^*(x(t))^\top D_x A(x(t), z^*(x(t))) \dot{x}(t), \\
    & \revision{= \lambda^*(x(t))^\top D_x A(x(t), z^*(x(t))) (f(x(t)) + g(x(t)) \tilde{u}(t))}, \\
    & \geq -\alpha h(x(t)).
\end{split}
\end{equation*}
This holds for almost all $t \in [t_0, T]$, i.e., the CBF constraint \cref{eq:ncbf-constraint} is satisfied.
By assumption, $x_0$ is such that $h(x_0) > 0$.
Finally, by \cref{lem:ncbf-safety}, $h(x(t)) > 0$ for all $t \in [t_0, T]$ \revision{and Filippov solutions}, i.e., \revision{the closed-loop system is strongly safe}.

\section*{Acknowledgments}
The authors gratefully acknowledge Dr. Somayeh Sojoudi at the University of California, Berkeley, for her valuable comments on the properties of parametric optimization problems.

\bibliographystyle{siamplain}
\bibliography{references}
\end{document}